\documentclass[12pt]{article}

\usepackage[sectionbib]{natbib}
\usepackage{array,epsfig,fancyheadings,rotating}
\usepackage{xcolor}
\usepackage{tikz}
\usepackage{pgfplots}
\pgfplotsset{compat=1.18}
\usepackage{hyperref}

\usepackage{comment}

\usepackage{amsmath}
\usepackage{amssymb}
\usepackage{amsfonts}
\usepackage{amsthm}
\usepackage{multirow}
\usepackage{sectsty,secdot}

\sectionfont{\fontsize{12}{14pt plus.8pt minus .6pt}\selectfont}
\subsectionfont{\fontsize{12}{14pt plus.8pt minus .6pt}\selectfont}

\newtheorem{theorem}{Theorem}
\newtheorem{lemma}{Lemma}
\newtheorem{corollary}{Corollary}

\theoremstyle{definition}

\newtheorem{remark}{Remark}
\newtheorem{prop}{Proposition}

\definecolor{hoeffding_col}{RGB}{0,114,178}
\definecolor{bernstein_col}{RGB}{213,94,0}
\definecolor{empbern_col}{RGB}{0,158,115}

\begin{document}

\renewcommand{\baselinestretch}{2}

\markright{ \hbox{\footnotesize\rm 
}\hfill\\[-13pt]
\hbox{\footnotesize\rm
}\hfill }

\markboth{\hfill{\footnotesize\rm Tareq Aldirawi, Yun Li AND Wenge Guo} \hfill}
{\hfill {\footnotesize\rm CRC under Non-Monotone Losses} \hfill}

\renewcommand{\thefootnote}{}
$\ $\par

\setcounter{section}{0}
\setcounter{equation}{0}
\setcounter{figure}{0}
\setcounter{table}{0}

\fontsize{12}{14pt plus.8pt minus .6pt}\selectfont
\vspace{0.8pc}
\centerline{\large\bf Conformal Risk Control under Non-Monotone Losses:}
\vspace{2pt}
\centerline{\large\bf Theory and Finite-Sample Guarantees}
\vspace{.4cm}
\centerline{Tareq Aldirawi, Yun Li and Wenge Guo}
\vspace{.4cm}
\centerline{\it New Jersey Institute of Technology}
\vspace{.55cm}

\fontsize{9}{11.5pt plus.8pt minus .6pt}\selectfont

\begin{quotation}
\noindent {\it Abstract:}
Conformal risk control (CRC) provides distribution-free guarantees for
controlling the expected loss at a user-specified level. Existing theory
typically assumes that the loss decreases monotonically with a tuning
parameter that governs the size of the prediction set. However, this
assumption is often violated in practice, where losses may behave
non-monotonically due to competing objectives such as coverage and
efficiency.

In this paper, we study CRC under non-monotone loss functions when the
tuning parameter is selected from a finite grid, a setting commonly
arising in thresholding and discretized decision rules. Revisiting a
known counterexample, we show that the validity of CRC without
monotonicity depends critically on the relationship between the
calibration sample size and the grid resolution. In particular, reliable
risk control can still be achieved when the calibration sample is
sufficiently large relative to the grid size.

We establish a finite-sample guarantee for bounded losses over a grid of
size $m$, showing that the excess risk above the target level $\alpha$
scales on the order of $\sqrt{\log(m)/n}$, where $n$ is the calibration
sample size. A matching lower bound demonstrates that this rate is
sharp for the CRC selection rule under the same population-safe
endpoint condition. We also derive refined guarantees under additional
structural conditions, including Lipschitz
continuity and monotonicity, and extend the analysis
to settings with distribution shift via importance weighting.

Numerical experiments on synthetic multilabel classification and real
object detection data illustrate the practical implications of
non-monotonicity. Methods that explicitly account for finite-sample
uncertainty achieve more stable risk control than approaches based on
monotonicity transformations, while maintaining competitive prediction
set sizes.

\vspace{9pt}
\par
\end{quotation}\par

\def\thefigure{\arabic{figure}}
\def\thetable{\arabic{table}}
\renewcommand{\theequation}{\thesection.\arabic{equation}}

\fontsize{12}{14pt plus.8pt minus .6pt}\selectfont

\section{Introduction}

Uncertainty quantification is essential in high-stakes prediction tasks, including medical diagnosis, autonomous systems, and scientific decision-making. While point predictions may achieve high accuracy, practitioners often require reliable measures of uncertainty. Conformal prediction addresses this need by constructing prediction sets with finite-sample coverage guarantees under the mild assumption of exchangeability \cite{vovk1999machine, vovk2005algorithmic, papadopoulos2002inductive}. However, coverage is not always the most relevant performance criterion. In many applications, practitioners instead seek to control other risk functionals, such as precision in information retrieval, false negative rates in medical screening, or fairness-related metrics in algorithmic decision-making.

Conformal risk control (CRC) \cite{angelopoulos2024conformal} extends the conformal prediction framework to provide distribution-free guarantees for controlling user-specified risks in expectation. We briefly review the CRC setup. Let
\[
D_{1:n+1} = ((X_1,Y_1),\ldots,(X_{n+1},Y_{n+1})) \in (\mathcal{X}\times\mathcal{Y})^{n+1}
\]
be exchangeable random variables with joint distribution $\mathbb{P}$. Given a calibration dataset $D_{1:n}$ and a test covariate $X_{n+1}$, the goal is to construct a prediction set for the unknown response $Y_{n+1}$.

Let $\mathcal{C}: \mathcal{X}\times\Lambda \to 2^{\mathcal{Y}}$ denote a prediction-set function indexed by $\lambda \in \Lambda \subseteq \mathbb{R}$. The parameter $\lambda$ controls the size of the prediction set, and we assume the monotonicity property
\[
\lambda_1 \le \lambda_2
\;\Rightarrow\;
\mathcal{C}(x;\lambda_1) \subseteq \mathcal{C}(x;\lambda_2)
\quad \text{for all } x \in \mathcal{X}.
\]
Let $\ell:\mathcal{X}\times\mathcal{Y}\times\Lambda \to [0,B]$ be a bounded loss function with $B<\infty$. The objective is to select $\hat{\lambda}\in\Lambda$ based on $D_{1:n}$ such that
\begin{equation}
\label{eq:risk-control}
\mathbb{E}\!\left[\ell(X_{n+1},Y_{n+1};\hat{\lambda})\right] \le \alpha,
\end{equation}
for a user-specified risk level $\alpha\in(0,1)$.

Under the additional assumption that $\ell(x,y;\lambda)$ is non-increasing in $\lambda$ for all $(x,y)$, CRC selects the parameter
\begin{equation}
\label{eq:crc-threshold}
\hat{\lambda}
=
\inf\!\left\{
\lambda\in\Lambda :
\frac{1}{n+1}
\left(
\sum_{i=1}^{n}\ell(X_i,Y_i;\lambda) + B
\right)
\le \alpha
\right\}.
\end{equation}
This construction accounts for the worst-case contribution of the test observation and yields the guarantee in \eqref{eq:risk-control} under exchangeability and monotonicity.

While this framework is general, the monotonicity assumption on the loss function may be restrictive in practice. For example, in object detection, controlling error rates requires balancing false positives and true detections, leading to losses that are typically non-monotonic in the score threshold $\lambda$. Similarly, in medical screening with multiple diagnostic criteria, expanding the decision rule may reduce some errors while increasing others, so the overall loss need not decrease with $\lambda$. When the loss is non-monotonic, the standard CRC guarantee may fail, as illustrated by a counterexample in \cite{angelopoulos2024conformal}.

In this paper, we revisit this setting and focus on the practically relevant case where the parameter space is a finite discrete grid
\[
\Lambda = \{\lambda_1,\ldots,\lambda_m\},
\]
as arises when tuning a score threshold or decision rule over a set of candidate values. We show that, even without monotonicity, CRC can still achieve approximate risk control when the calibration sample size is sufficiently large relative to the grid size. In particular, our main finite-sample guarantee shows that
\[
\mathbb{E}\!\left[\ell(X_{n+1},Y_{n+1};\hat{\lambda})\right]
\le
\alpha + C\sqrt{\frac{\log m}{n}},
\]
for a universal constant $C>0$. We further establish a matching lower bound showing that this rate is sharp for the CRC selection rule under the same population-safe endpoint condition.

We also investigate structural conditions under which improved guarantees can be obtained. 
In particular, Lipschitz continuity yields an excess-risk bound
independent of the original grid size, whereas monotone losses
recover the standard CRC guarantee.
Finally, we extend our analysis to
settings with distribution shift through importance weighting.

These results show that monotonicity is not essential for effective conformal risk control in discretized settings, provided that the calibration sample size is sufficiently large relative to the parameter grid.

\paragraph{Contributions.}
Our main contributions are as follows.
\begin{enumerate}
\item We study conformal risk control when the loss function is non-monotonic in the prediction set parameter. Revisiting a counterexample from the CRC literature, we show that the behavior of CRC depends critically on the interplay between the calibration sample size and the resolution of the parameter grid.

\item For bounded losses over a finite grid $\Lambda=\{\lambda_1,\ldots,\lambda_m\}$, we establish a finite-sample guarantee showing that the excess risk above the target level $\alpha$ is of order $\sqrt{\log(m)/n}$.

\item We prove a matching lower bound demonstrating that the $\sqrt{\log(m)/n}$ rate is sharp for the CRC selection rule under the same population-safe endpoint condition.

\item We obtain refined guarantees under structural assumptions on the loss function, including Lipschitz and monotone losses.

\item We extend the framework to settings with distribution shift using importance weighting and provide empirical evaluations on synthetic and real-world benchmarks.
\end{enumerate}

\paragraph{Organization of the paper.}
Section~\ref{sec:related} reviews related work.
Section~\ref{sec:counterexample} revisits the counterexample illustrating the failure of CRC under non-monotonic losses.
Sections~\ref{sec:generalization}--\ref{subsec:lower-bound} develop our main finite-sample guarantees together with the corresponding rate-sharpness result for CRC.
Section~\ref{sec:structure} studies refined guarantees under structural assumptions on the loss function, while Section~\ref{sec:existing-comparison} compares our approach with existing methods for handling non-monotonic losses.
Section~\ref{sec:example} presents empirical results, and
Section~\ref{sec:distributional_shift} extends the framework to settings with distribution shift.
Finally, Section~\ref{sec:conc} concludes with a summary of the main findings.

All technical proofs, auxiliary lemmas, and additional results are provided in the online supplementary material, including detailed derivations for Section~\ref{sec:counterexample}, extended comparisons for Section~\ref{sec:existing-comparison}, and full experimental descriptions for Section~\ref{sec:example}, along with refined concentration bounds and additional numerical results.

\section{Related Work}
\label{sec:related}

Conformal prediction, introduced by Vovk, Gammerman, and collaborators
\cite{vovk1999machine, vovk2005algorithmic}, provides model-agnostic
uncertainty quantification with finite-sample coverage guarantees under
exchangeability. Split conformal prediction
\cite{papadopoulos2002inductive} offers a computationally efficient
variant based on a held-out calibration set, with statistical properties
studied in \cite{lei2018distribution}. Comprehensive introductions can
be found in \cite{angelopoulos2021gentle, shafer2008tutorial}, and recent
developments are surveyed in \cite{angelopoulos2024theoretical}. Extensions
beyond exchangeability
\cite{tibshirani2019conformal, podkopaev2021distribution, barber2023conformal}
and results on conditional coverage
\cite{vovk2012conditional, foygel2021limits, gibbs2025conformal}
further broaden the scope of conformal inference.

Building on this literature, recent work has developed conformal methods
for controlling more general risk functionals. Risk-controlling prediction
sets (RCPS) \cite{bates2021distribution} provide high-probability bounds
on user-specified risks under monotone loss functions. The learn-then-test
(LTT) framework \cite{angelopoulos2025learn} introduces a modular two-stage
procedure based on multiple testing, which can accommodate non-monotone
loss functions. Conformal risk control (CRC) \cite{angelopoulos2024conformal}
offers an alternative conformal framework that controls risk in expectation
for monotone losses. Compared with RCPS, which provides high-probability
guarantees, CRC controls risk in expectation and can therefore yield less
conservative procedures. Standard split conformal prediction can be viewed
as a special case of CRC. Several extensions of CRC have subsequently been
proposed, including cross-validation and leave-one-out variants
\cite{cohen2024cross}, full conformal risk control
\cite{angelopoulos2024note}, and weighted methods for non-exchangeable
settings \cite{farinhas2023non}. Applications to modern machine learning
tasks include \cite{schuster2021consistent, fisch2022conformal,
angelopoulos2022image, feldman2022achieving, teneggi2023trust}. Recently,
Angelopoulos \cite{angelopoulos2026conformal} proposed stability-based
approaches for handling non-monotonicity in CRC by imposing structural
assumptions on the learning algorithm rather than on the loss function.

Our work provides a distinct perspective on conformal risk control under
non-monotone losses. Rather than relying on algorithmic stability, we
analyze the statistical error induced by selecting the prediction-set
parameter over a finite grid. For bounded loss functions, including
non-monotone ones, we establish finite-sample guarantees that explicitly
control the excess risk above the target level.

In contrast to LTT, our guarantees are formulated in expectation, and
unlike standard CRC, our analysis does not require monotonicity of the
loss function. We show that the CRC excess-risk upper bound scales as
$\sqrt{\log(m)/n}$ and establish a matching rate lower bound for the
CRC selector under the same population-safe endpoint condition.
This sharpness result concerns the CRC rule, not minimax optimality
over all measurable selection procedures.

These results provide a principled approach to handling non-monotonicity
in CRC, showing that it can be addressed through statistical analysis of
the selection procedure rather than through structural assumptions on
the loss or the learning algorithm.

\section{Main Results}\label{sec:theory}
We study conformal risk control under non-monotonic losses.
Section~\ref{sec:counterexample} revisits the counterexample from
\cite{angelopoulos2024conformal} and sharpens its implications.
Section~\ref{sec:generalization} develops finite-sample expectation bounds
for discrete parameter sets, while
Section~\ref{subsec:lower-bound} establishes complementary lower bounds.

\subsection{Counterexample Analysis}
\label{sec:counterexample}

Write $L_i(\lambda):=\ell(X_i,Y_i;\lambda)$. 
We revisit the counterexample from \cite{angelopoulos2024conformal} and
highlight its scaling behavior as the grid size $m$ grows.

Let $B=1$, $\Lambda=\{0,1/m,\dots,1\}$,
$\alpha\in(1/(n+1),1)$, and $p\in(\alpha,1)$.
For $j<m$, define $L_i(j/m)\overset{\text{i.i.d.}}{\sim}\mathrm{Bern}(p)$
and $L_i(1)=0$, independent across $i$ and $j$.
Thus the loss is non-monotonic since
$\mathbb{E}[L_i(j/m)]=p$ for $j<m$ but $\mathbb{E}[L_i(1)]=0$.

The CRC threshold~\eqref{eq:crc-threshold} reduces to
\[
\hat{\lambda}
=
\inf\Bigl\{
\lambda\in\Lambda:
\frac{1}{n+1}\sum_{i=1}^n L_i(\lambda)
+\frac{1}{n+1}
\le\alpha
\Bigr\},
\]
with $\inf\varnothing=1$.
Since $\hat{\lambda}$ depends only on $D_{1:n}$, it is independent of $L_{n+1}(\cdot)$.
Because $L_{n+1}(1)=0$ and $L_{n+1}(j/m)\sim\mathrm{Bern}(p)$ for $j<m$,
\begin{equation}
\mathbb{E}[L_{n+1}(\hat{\lambda})]
=
p\,\mathbb{P}(\hat{\lambda}\neq1).
\label{eq:counterexample-risk}
\end{equation}
Let
\[
q=\mathbb{P}\!\left(\frac1n\sum_{i=1}^n Z_i
\le \alpha-\frac1{n+1}\right),
\quad Z_i\sim\mathrm{Bern}(p).
\]
Independence across grid points implies
\[
\mathbb{P}(\hat{\lambda}=1)=(1-q)^m,
\]
so that
\begin{equation}
\mathbb{E}[L_{n+1}(\hat{\lambda})]
=
p\bigl(1-(1-q)^m\bigr).
\label{eq:counterexample-main}
\end{equation}

Equation~\eqref{eq:counterexample-main} shows that the excess risk is governed by the interaction between the calibration sample size $n$ (through $q$) and the grid size $m$.
Standard concentration bounds for $q$ yield upper and lower bounds on $\mathbb{E}[L_{n+1}(\hat{\lambda})]$ that reveal a sharp scaling transition.

\medskip
\noindent
\textbf{Risk is controlled} when $n$ is sufficiently large relative to $\log m$, for example when
\[
n \gtrsim \frac{\log m}{(p-\alpha)^2},
\]
in which case $\mathbb{E}[L_{n+1}(\hat{\lambda})] \le \alpha$.

\medskip
\noindent
\textbf{Risk is not controlled} when $n$ is too small compared to $\log m$, for instance when
\[
n \lesssim \frac{\log m}{-\log(1-p)},
\]
in which case $\mathbb{E}[L_{n+1}(\hat{\lambda})] > \alpha$.

\medskip
\noindent
Thus, without monotonicity, risk control is no longer universal but becomes \emph{scaling-dependent}, determined by the relationship between the sample size $n$ and the grid resolution $m$.
Increasing $m$ enlarges the selection space and can invalidate risk control unless supported by a sufficiently large calibration sample.

\medskip
\noindent
\textbf{Remark.}
Detailed derivations, including concentration bounds for $q$ and the resulting inequalities in~\eqref{eq:counterexample-main}, are provided in the online supplementary material.

\subsection{Main Result for Non-Monotonic Losses}
\label{sec:generalization}
We derive an upper bound on the expected non-monotonic losses over a
finite parameter set $\Lambda=\{\lambda_1,\dots,\lambda_m\}$.
By appropriately adjusting the pre-specified risk level, this leads to
finite-sample expectation control for non-monotonic losses.

Throughout this section, we assume that for each $\lambda \in \Lambda$,
the sequence $\{L_i(\lambda)\}_{i=1}^{n+1}$ is i.i.d.\
The analysis relies on Hoeffding's inequality.
Let
\[
R(\lambda)=\mathbb{E}[L_i(\lambda)], 
\quad
\hat R_n(\lambda)=\frac{1}{n}\sum_{i=1}^n L_i(\lambda).
\]

\begin{theorem}[Risk control under non-monotone losses]
\label{thm:nonmonotonic}
Let $\Lambda=\{\lambda_1<\cdots<\lambda_m\}$ be a finite
parameter grid, and let $0\le\alpha\le B$.
Suppose that the loss vectors
$\{(L_i(\lambda))_{\lambda\in\Lambda}\}_{i=1}^{n+1}$
are i.i.d. and satisfy
\[
0\le L_i(\lambda)\le B
\quad\text{almost surely for all }\lambda\in\Lambda.
\]
Assume that the terminal parameter $\lambda_m$
satisfies the population-risk condition
\[
R(\lambda_m)\le\alpha.
\]

Define
\[
\hat{\lambda}
=
\inf\left\{
\lambda\in\Lambda:
\frac{n}{n+1}\hat R_n(\lambda)
+\frac{B}{n+1}\le\alpha
\right\},
\]
with the convention $\inf\varnothing=\lambda_m$.
Then
\[
\mathbb{E}\bigl[L_{n+1}(\hat{\lambda})\bigr]
\le\alpha+D(m,n),
\]
where
\[
D(m,n)
=
B\sqrt{\frac{\log(2m)}{2n}}
+
\frac{B}{2\sqrt{2n\log(2m)}}.
\]

Furthermore, let
\[
\alpha'=\alpha-D(m,n)>0
\]
and suppose that the terminal parameter satisfies
the stronger population-risk condition
\[
R(\lambda_m)\le\alpha'.
\]
Define the adjusted selector
\[
\hat{\lambda}^{\mathrm{adj}}
=
\inf\left\{
\lambda\in\Lambda:
\frac{n}{n+1}\hat R_n(\lambda)
+\frac{B}{n+1}\le\alpha'
\right\},
\]
again with $\inf\varnothing=\lambda_m$.
Then
\[
\mathbb{E}\bigl[
L_{n+1}(\hat{\lambda}^{\mathrm{adj}})
\bigr]
\le\alpha'+D(m,n)=\alpha.
\]
\end{theorem}

The proof (Appendix~\ref{appendix:B}) is based on a uniform concentration argument 
over $\Lambda$ using Hoeffding's inequality, combined with a union bound over the $m$ candidate parameters.

\begin{remark}
The population-risk condition $R(\lambda_m)\le\alpha$
requires only that the terminal grid point $\lambda_m$
have expected loss no greater than the target level.
This assumption is weaker than the almost-sure
safe-endpoint condition used in standard conformal
risk control \citep{angelopoulos2024conformal},
as it allows $L_i(\lambda_m)$ to exceed $\alpha$
for individual observations. In particular,
the loss at $\lambda_m$ need not be zero or
uniformly bounded by $\alpha$.
\end{remark}

\noindent
\textbf{Implications and Interpretation.}
The theorem provides a finite-sample expectation guarantee for
non-monotonic losses, with an explicit excess term $D(m,n)$ capturing
the statistical cost of model selection over a finite parameter set.

A fundamental trade-off emerges: increasing $m$ improves resolution
by reducing discretization error, but simultaneously enlarges the
selection space and thus increases estimation uncertainty.
Since this cost grows only logarithmically in $m$, moderately rich
parameter grids remain effective when the sample size $n$ is sufficiently
large.

The guarantee requires only bounded losses, independent
and identically distributed loss vectors, and the
population-risk condition $R(\lambda_m)\le\alpha$.
Under these assumptions, the selected parameter
$\hat\lambda$ achieves expected loss at most
$\alpha+D(m,n)$, without requiring monotonicity
of the loss function.

The bound $D(m,n)$ in Theorem~\ref{thm:nonmonotonic} is derived
using Hoeffding's inequality, which depends on the range of the
loss rather than its variance.
When the loss distribution has low variance (i.e., is concentrated around its mean), this can be conservative.
In Appendix~\ref{appendix:tighter_bounds} we develop
variance-sensitive concentration inequalities that exploit the
empirical variance of the losses, yielding tighter corrections
when the loss distribution is concentrated.

Thus, risk control in the non-monotonic setting becomes
resolution-dependent: finer parameter grids require larger sample
sizes to maintain comparable statistical reliability.
\begin{table}[t]
\centering
\caption{Excess term $D(m,n)$ for $m=100$ and $m=200$ ($B=1$).}
\label{tab:excess_combined}
\begin{tabular}{c|ccc|ccc}
\hline
& \multicolumn{3}{c|}{$m=100$} 
& \multicolumn{3}{c}{$m=200$} \\
$n$ 
& $D$ & $D/0.1$ & $D/0.2$
& $D$ & $D/0.1$ & $D/0.2$ \\
\hline
1{,}000   & 0.0563 & 0.5630 & 0.2815 & 0.0606 & 0.6060 & 0.3030 \\
5{,}000   & 0.0252 & 0.2520 & 0.1260 & 0.0271 & 0.2710 & 0.1355 \\
10{,}000  & 0.0178 & 0.1780 & 0.0890 & 0.0192 & 0.1920 & 0.0960 \\
50{,}000  & 0.0080 & 0.0800 & 0.0400 & 0.0086 & 0.0860 & 0.0430 \\
100{,}000 & 0.0056 & 0.0560 & 0.0280 & 0.0061 & 0.0610 & 0.0305 \\
\hline
\end{tabular}
\end{table}
Table~\ref{tab:excess_combined} illustrates the theoretical behavior of
$D(m,n)$. The excess term decays at the expected $1/\sqrt{n}$ rate and
depends only logarithmically on the grid size.
Doubling the grid from $m=100$ to $m=200$ changes the bound only slightly,
consistent with its logarithmic dependence on $m$.

More importantly, the inflation rate $D/\alpha$ becomes small once the
calibration sample reaches a few thousand observations.
For instance, when $\alpha = 0.2$ and $m = 100$,
$n = 10{,}000$ yields $D \approx 0.0178$, corresponding to
$D/\alpha \approx 0.089$.
At $n = 50{,}000$, the inflation rate drops to about $D/\alpha \approx 0.04$.
Thus, the theoretical correction translates into only modest
practical inflation at realistic sample sizes.

\paragraph{Discretization error.}
Let $\lambda_{\mathrm{ora}}=\inf\{\lambda : R(\lambda)\le\alpha\}$
denote the oracle threshold in the underlying continuous parameter
space. If $\lambda_{\mathrm{ora}}\notin\Lambda$, the grid-restricted
oracle
\[
\lambda_m^\star=\inf\{\lambda\in\Lambda : R(\lambda)\le\alpha\}
\]
satisfies $\lambda_m^\star\ge\lambda_{\mathrm{ora}}$.
Consequently, the selected parameter is necessarily more conservative
than the true oracle, potentially resulting in larger prediction sets
or higher loss.

This gap represents an intrinsic approximation error induced by the
finite grid. As $m$ increases, the grid becomes denser and
$\lambda_m^\star$ moves closer to $\lambda_{\mathrm{ora}}$,
thereby reducing discretization bias.
At the same time, expanding $\Lambda$ increases the complexity of the
selection problem, which in turn amplifies stochastic variability.
This effect is captured by the excess term $D(m,n)$.

A fundamental discretization--selection trade-off therefore emerges:
finer grids improve approximation to the continuous oracle but make
reliable estimation more challenging, requiring larger sample sizes to
control the additional variability.
Thus, the choice of $\Lambda$ must balance model resolution against
statistical stability.

\subsection{Sharpness of the CRC Excess-Risk Bound}
\label{subsec:lower-bound}

Theorem~\ref{thm:nonmonotonic} establishes a
finite-sample risk bound for CRC under
non-monotone losses, with an excess-risk term
of order $\sqrt{\log(m)/n}$. A natural question
is whether this dependence is intrinsic to
the CRC selection rule or merely an artifact
of the uniform concentration argument.
The following result establishes rate
sharpness under the same population-safe
endpoint condition as in
Theorem~\ref{thm:nonmonotonic}.

\begin{prop}[Sharpness of the CRC excess-risk bound]
\label{prop:sharpness}
Fix $\alpha\in(0,1)$ and consider binary
losses with $B=1$. There exist constants
$c_\alpha>0$ and $N_\alpha\ge1$, depending
only on $\alpha$, such that, for every
$n\ge N_\alpha$ and $m\ge4$, there exists
a distribution satisfying all assumptions
of Theorem~\ref{thm:nonmonotonic}, with
\[
R(\lambda_m)=\alpha,
\]
for which the CRC selector $\hat\lambda$
satisfies
\begin{equation}
\label{eq:sharpness-lower}
\mathbb E\bigl[L_{n+1}(\hat\lambda)\bigr]
\ge
\alpha+
c_\alpha
\min\left\{
1,\sqrt{\frac{\log m}{n}}
\right\}.
\end{equation}
Consequently, the dependence on
$\sqrt{\log(m)/n}$ in
Theorem~\ref{thm:nonmonotonic}
is rate-sharp for the CRC selector
when $\log m\le n$.
\end{prop}

Proposition~\ref{prop:sharpness} shows that
the excess-risk term in
Theorem~\ref{thm:nonmonotonic} cannot
generally be eliminated for CRC.
Even when the terminal parameter is
population-safe, downward fluctuations
in empirical risks may lead CRC to
select an unsafe parameter. With
multiple candidate parameters, the
probability of such a selection can
remain bounded away from zero, resulting
in excess risk of order
$\sqrt{\log(m)/n}$ when $\log m\le n$.
The lower bound holds for every finite
$n\ge N_\alpha$, with constants independent
of $n$ and $m$, and establishes rate
sharpness for the CRC selector rather
than minimax optimality over all
measurable selection rules.

\section{Improved Risk Bounds under Structured Loss Functions}
\label{sec:structure}

Theorem~\ref{thm:nonmonotonic} establishes a finite-sample
risk guarantee for general bounded, non-monotone losses,
with an excess-risk term of order
$\mathcal{O}(\sqrt{\log m/n})$. This term arises from
uniform concentration over the finite class
$\{L(\cdot;\lambda):\lambda\in\Lambda\}$, without
requiring additional structural assumptions on
the loss function.

Sharper guarantees can be obtained by exploiting
additional regularity. In particular, Lipschitz
continuity in the tuning parameter yields an
improved risk bound whose complexity term is
independent of the number of grid points, while
monotonicity recovers the standard CRC guarantee.
These results illustrate how structural properties
of the loss function can reduce or eliminate
the excess risk associated with data-dependent
parameter selection.

\subsection{Improved Risk Guarantees under Lipschitz Continuity}

Theorem~\ref{thm:nonmonotonic} provides a finite-sample
risk guarantee for general bounded, non-monotone losses.
However, its excess-risk bound $D(m,n)$ depends on the
number $m$ of grid points and may become conservative
when the selection grid is very fine. Under additional
regularity, this dependence can be reduced or eliminated.

We next consider the case where the loss function
is Lipschitz continuous in the tuning parameter.
Suppose that the assumptions of
Theorem~\ref{thm:nonmonotonic} hold, and let
$\mathcal I=[a,b]$, with $a<b$, be a bounded
interval containing $\Lambda$.
Assume further that the loss function extends
to $\mathcal I$, remains bounded by $B$, and
satisfies
\[
|L(Z;\lambda)-L(Z;\lambda')|
\le K|\lambda-\lambda'|,
\qquad
\forall\,\lambda,\lambda'\in\mathcal I,
\]
almost surely, where $K<\infty$ is independent
of the selection grid $\Lambda$.

The following proposition shows that this additional
regularity yields a risk bound whose complexity term
is independent of the number of selection grid points.

\begin{prop}[Improved risk guarantee under Lipschitz continuity]
\label{thm:lipschitz}
Suppose that the assumptions of
Theorem~\ref{thm:nonmonotonic} hold, together with
the uniform Lipschitz conditions stated above.

Then the selector $\hat\lambda$ satisfies
\[
\mathbb E[L_{n+1}(\hat\lambda)]
\le
\alpha+\widetilde D(n,K,b-a),
\]
where
\[
\widetilde D(n,K,b-a)
=
\inf_{N\ge2}
\left\{
D(N,n)+\frac{K(b-a)}{N-1}
\right\},
\]
the infimum being taken over integers $N\ge2$, and
\[
D(N,n)
=
B\sqrt{\frac{\log(2N)}{2n}}
+
\frac{B}{2\sqrt{2n\log(2N)}}.
\]

Consequently,
\[
\mathbb E[L_{n+1}(\hat\lambda)]
\le
\alpha+
\min\left\{
D(m,n),\,
\widetilde D(n,K,b-a)
\right\}.
\]
\end{prop}

\begin{remark}[Implications of Lipschitz continuity]
Proposition~\ref{thm:lipschitz} shows that Lipschitz
continuity can yield a sharper finite-sample risk
guarantee than Theorem~\ref{thm:nonmonotonic},
particularly when the selection grid is very fine.
Unlike $D(m,n)$, the Lipschitz-based bound
$\widetilde D(n,K,b-a)$ does not depend on the
number of selection grid points $m$.

This improvement is obtained by controlling the
uniform deviation between empirical and expected
losses through an auxiliary grid whose size is
chosen independently of the selection grid.
Consequently, refining the selection grid does
not necessarily weaken the risk guarantee.

Moreover, for fixed $B$, $K$, and $b-a$, choosing
$N\asymp\sqrt{n}$ yields
\[
\widetilde D(n,K,b-a)
=
\mathcal O\left(\sqrt{\frac{\log n}{n}}\right).
\]
Thus, the excess-risk bound converges to zero
even when the number of selection grid points
increases with the calibration sample size,
provided the stated regularity conditions
continue to hold.
\end{remark}

\subsection{Exact Risk Control under Monotonicity}

We next consider the stronger structural assumption
of monotonicity, under which the standard CRC
guarantee is recovered without an excess-risk term.

\begin{prop}[Exact risk control under monotonicity]
\label{thm:monotone}
Suppose that the assumptions of
Theorem~\ref{thm:nonmonotonic} hold.
Assume further that $L_i(\lambda)$ is
non-increasing in $\lambda$ and that
$L_i(\lambda_m)\le\alpha$ almost surely,
for $i=1,\ldots,n+1$.
Then
\[
\mathbb E[L_{n+1}(\hat\lambda)]\le\alpha.
\]
\end{prop}

\begin{proof}
As in the proof of
Theorem~\ref{thm:nonmonotonic}, define
\[
\hat\lambda'
=
\inf\left\{
\lambda\in\Lambda:
\hat R_{n+1}(\lambda)\le\alpha
\right\}.
\]
The almost-sure endpoint condition ensures
that the feasible set is nonempty.

Since
\[
\hat R_{n+1}(\lambda)
\le
\frac{n}{n+1}\hat R_n(\lambda)
+\frac{B}{n+1},
\]
the definitions of the two selectors imply
\[
\hat\lambda'\le\hat\lambda
\quad\text{almost surely}.
\]

Now consider the difference
\[
\Delta_n
=
\mathbb E[L_{n+1}(\hat\lambda)]
-
\mathbb E[L_{n+1}(\hat\lambda')].
\]
Without monotonicity, controlling this
difference requires a uniform deviation
bound, leading to the additional term
$D(m,n)$ in
Theorem~\ref{thm:nonmonotonic}.

Under the present monotonicity assumption,
however,
\[
L_{n+1}(\hat\lambda)
\le
L_{n+1}(\hat\lambda')
\quad\text{almost surely},
\]
and hence $\Delta_n\le0$.

Finally, exchangeability and the
definition of $\hat\lambda'$ give
\[
\mathbb E[L_{n+1}(\hat\lambda')]
=
\mathbb E[\hat R_{n+1}(\hat\lambda')]
\le\alpha.
\]
Therefore,
\[
\begin{aligned}
\mathbb E[L_{n+1}(\hat\lambda)]
&=
\mathbb E[L_{n+1}(\hat\lambda')]
+\Delta_n\\
&\le\alpha.
\end{aligned}
\]
Thus, monotonicity eliminates the
additional term $D(m,n)$.
\end{proof}

This result recovers the standard CRC guarantee
under monotone losses and the almost-sure
safe-endpoint condition. In particular, the
excess-risk term $D(m,n)$ appearing in
Theorem~\ref{thm:nonmonotonic} is eliminated.
However, exact risk control does not imply
that the selected parameter is optimal over
a continuous parameter space. As discussed
in Section~\ref{sec:generalization}, restricting
selection to a finite grid may introduce
additional conservativeness when the continuous
oracle $\lambda_{\mathrm{ora}}$ does not belong
to $\Lambda$. Thus, exact control refers to
the absence of statistical excess over the
target risk level $\alpha$, rather than the
absence of discretization error.

Together, Theorem~\ref{thm:nonmonotonic} and
Propositions~\ref{thm:lipschitz} and
\ref{thm:monotone} establish a hierarchy of
finite-sample risk guarantees. For general
bounded, non-monotone losses, the excess-risk
bound is of order
$\mathcal O(\sqrt{\log m/n})$
under the population-risk endpoint condition.
Under Lipschitz continuity, this bound can be
improved to a complexity term independent of
the number of selection grid points.
Under monotonicity and the stronger almost-sure
safe-endpoint condition, the excess-risk term
disappears entirely, recovering exact
$\alpha$-level risk control.

\section{Comparison with Existing Methods for Handling Non-Monotonic Loss}
\label{sec:existing-comparison}

We compare our approach with two existing strategies for handling
non-monotonic losses in conformal risk control (CRC):
monotonization techniques \citep{angelopoulos2024conformal}
and stability-based methods \citep{angelopoulos2026conformal}.
These approaches differ fundamentally in how they restore validity when
the monotonicity assumption fails.

\subsection{Monotonization Techniques}

Monotonization methods enforce monotonicity by modifying either the loss
function or the empirical risk.
Loss monotonization replaces each loss by its worst-case value over
larger thresholds, while risk monotonization applies a similar envelope
operation to the empirical risk.
Both approaches reduce the problem to the standard monotone CRC framework.

These methods involve a trade-off.
Loss monotonization provides finite-sample validity but may be overly
conservative due to worst-case inflation.
Risk monotonization is typically less conservative but offers only
asymptotic guarantees.

In contrast, our method leaves both the loss function and the empirical
risk unchanged.
Instead, it treats non-monotonicity as a finite model-selection problem
over the grid $\Lambda=\{\lambda_1,\dots,\lambda_m\}$.
The resulting correction $D(m,n)$ is explicit, rate-sharp for the CRC
selector, and vanishes as $n\to\infty$.
Exact control is achieved by selecting at the adjusted
level $\alpha'=\alpha-D(m,n)$, provided that
$R(\lambda_m)\le\alpha'$, as required by Theorem~\ref{thm:nonmonotonic}.

\begin{figure}[t]
  \centering
  \includegraphics[width=0.6\textwidth]{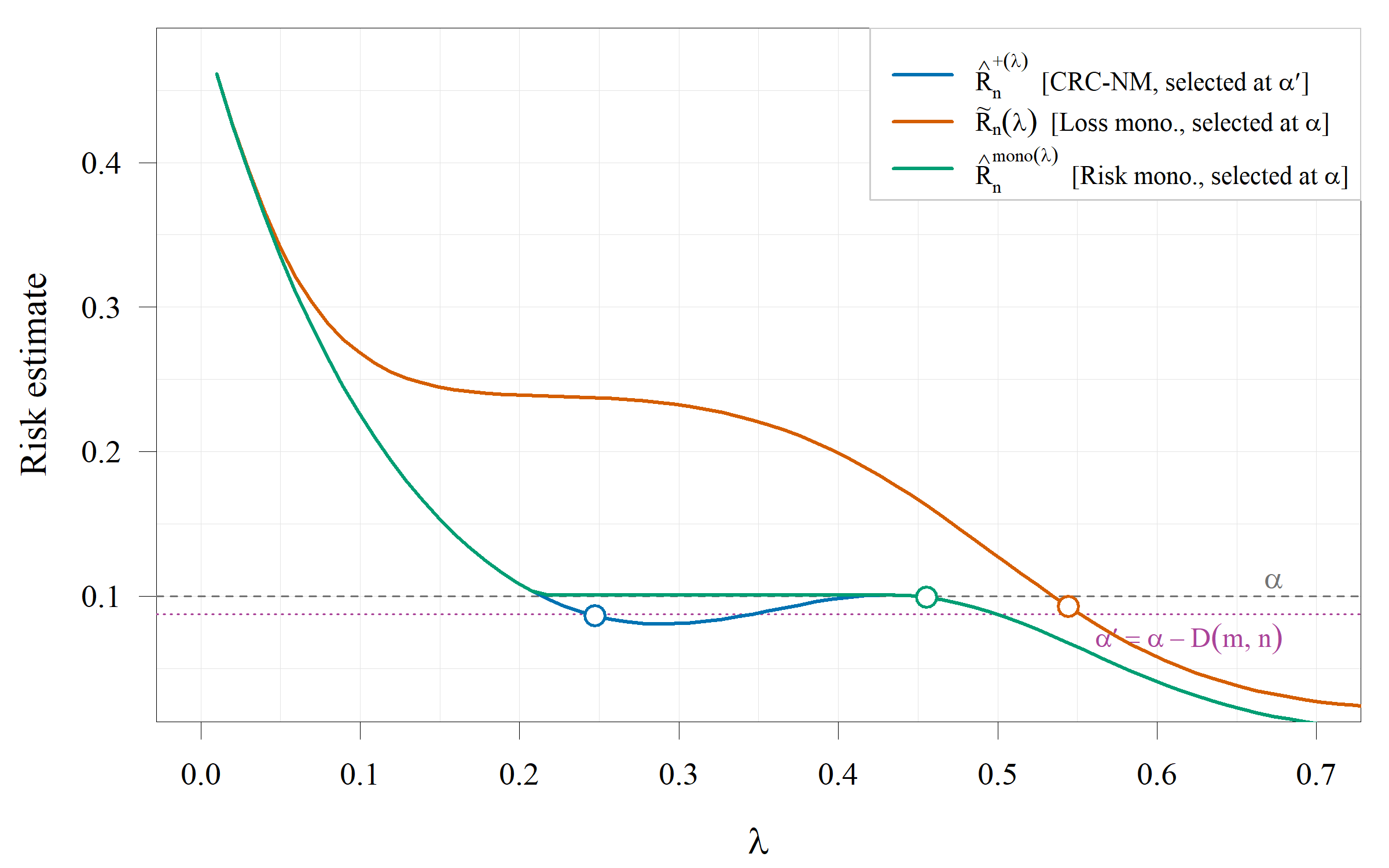}
  \caption{Empirical risk curves and selected thresholds under a
  non-monotonic loss.
  Monotonization methods (orange and green) enforce monotonicity via
  right-envelope corrections and select at level $\alpha$.
  CRC-NM selects at the adjusted level $\alpha'=\alpha-D(m,n)$ and
  operates directly on the original risk curve (blue), often yielding
  less conservative thresholds.}
  \label{fig:monotonization-main}
\end{figure}

Figure~\ref{fig:monotonization-main} illustrates the behavior of the
three approaches under a non-monotonic risk curve.
Monotonization methods enforce monotonicity via right-envelope
corrections, which propagate local irregularities across the parameter
range and may restrict feasible choices.
In contrast, CRC-NM operates directly on the original risk curve and
selects from regions that remain feasible after accounting for the
finite-sample correction.

This comparison highlights a key distinction:
monotonization incurs a \emph{structural} cost by modifying the loss or
risk, whereas CRC-NM incurs a \emph{statistical} cost $D(m,n)$ arising
from selection over a finite grid.

\subsection{Stability-Based Approaches}

Stability-based methods \citep{angelopoulos2026conformal} control risk by
comparing a data-driven algorithm with an oracle through a stability
parameter $\beta$.
In practice, this leads to an adjusted level
$\alpha'=\alpha-\hat{\beta}$, where $\hat{\beta}$ is typically estimated
via bootstrap.

When the stability constant is small and accurately estimated,
this approach can yield tight, data-adaptive corrections.
However, it requires additional assumptions and careful estimation, and
may be sensitive to the bootstrap procedure.

By contrast, our approach provides an explicit finite-sample bound that
depends only on $(n,m,B,\alpha)$ and requires no stability assumptions
or resampling.
The correction $D(m,n)$ controls the uniform deviation over all
candidate parameters and ensures valid risk control under minimal
conditions.

\medskip
\noindent
\textbf{Summary.}
Monotonization methods ensure validity through structural modification
of the loss or risk, while stability-based methods rely on data-adaptive
corrections.
Our approach instead provides a direct statistical treatment of
non-monotonicity, yielding an explicit, rate-sharp correction that
vanishes with increasing sample size.

\medskip
\noindent
\textbf{Remark.}
Detailed formulations of the monotonization and stability-based methods,
together with additional empirical comparisons, are provided in the
online supplementary material.

\section{Examples} \label{sec:example}

We evaluate the proposed method from Theorem~\ref{thm:nonmonotonic}
in representative settings. We compare standard CRC, the
bootstrap-corrected variant CRC-C \cite{angelopoulos2026conformal},
the loss and risk monotonization approaches (Loss Mono and Risk Mono)
\cite{angelopoulos2024conformal}, and the proposed CRC-NM.

Throughout, CRC and CRC-C operate at levels $\alpha$ and
$\alpha'=\alpha-\hat{\beta}$, respectively, while CRC-NM applies the
finite-sample correction $D(m,n)$ directly.
Full experimental details are provided in the online supplementary material.

\subsection{Synthetic Multilabel Experiment}
\label{sec:synthetic}

We first consider a synthetic multilabel classification
experiment designed to generate non-monotone risk curves.
Covariates are generated from
$X\sim\mathcal{N}(0,I_d)$ with $d=15$.
Conditional on $X$, the binary labels $Y_k$ are
independently generated with probabilities
$\sigma(X^\top W_k+b_k)$, where $\sigma$ denotes
the logistic function. Prediction sets
$C(X;\lambda)$ are constructed by thresholding
noisy estimates of the label probabilities.

To induce non-monotonicity, we consider the loss
\[
L(X,Y;\lambda)
=
c
+
w\left(
1-\frac{|Y\cap C(X;\lambda)|}{|C(X;\lambda)|}
\right)
+
a\,s_\lambda
\frac{1+\sin(4\pi s_\lambda)}{2},
\qquad
s_\lambda=\frac{|C(X;\lambda)|}{K},
\]
where precision is defined to be one when
$C(X;\lambda)=\varnothing$.
We set $c=0.10$, $w=0.25$, and $a=0.50$.
The resulting loss is bounded in
$[0.10,0.85]\subset[0,1]$ and is generally
non-monotone in the prediction-set size.

At the terminal threshold $\lambda_m=1$,
the prediction set is empty, yielding
\[
L(X,Y;\lambda_m)=0.10
\]
for every observation. Consequently,
\[
R(\lambda_m)=0.10\le\alpha=0.15,
\]
so the population-risk endpoint condition
in Theorem~\ref{thm:nonmonotonic} is satisfied.
In fact, the stronger almost-sure endpoint
condition also holds in this experiment.

We use $K=10$, a calibration sample of size
$n_{\mathrm{cal}}=10{,}000$, a test sample of size
$n_{\mathrm{test}}=500$, and a selection grid
containing $m=100$ points. The target risk level
is $\alpha=0.15$, and results are summarized
over $400$ independent repetitions.

\begin{figure}[t]
\centering
\includegraphics[height=0.2\textheight,width=0.48\textwidth]
{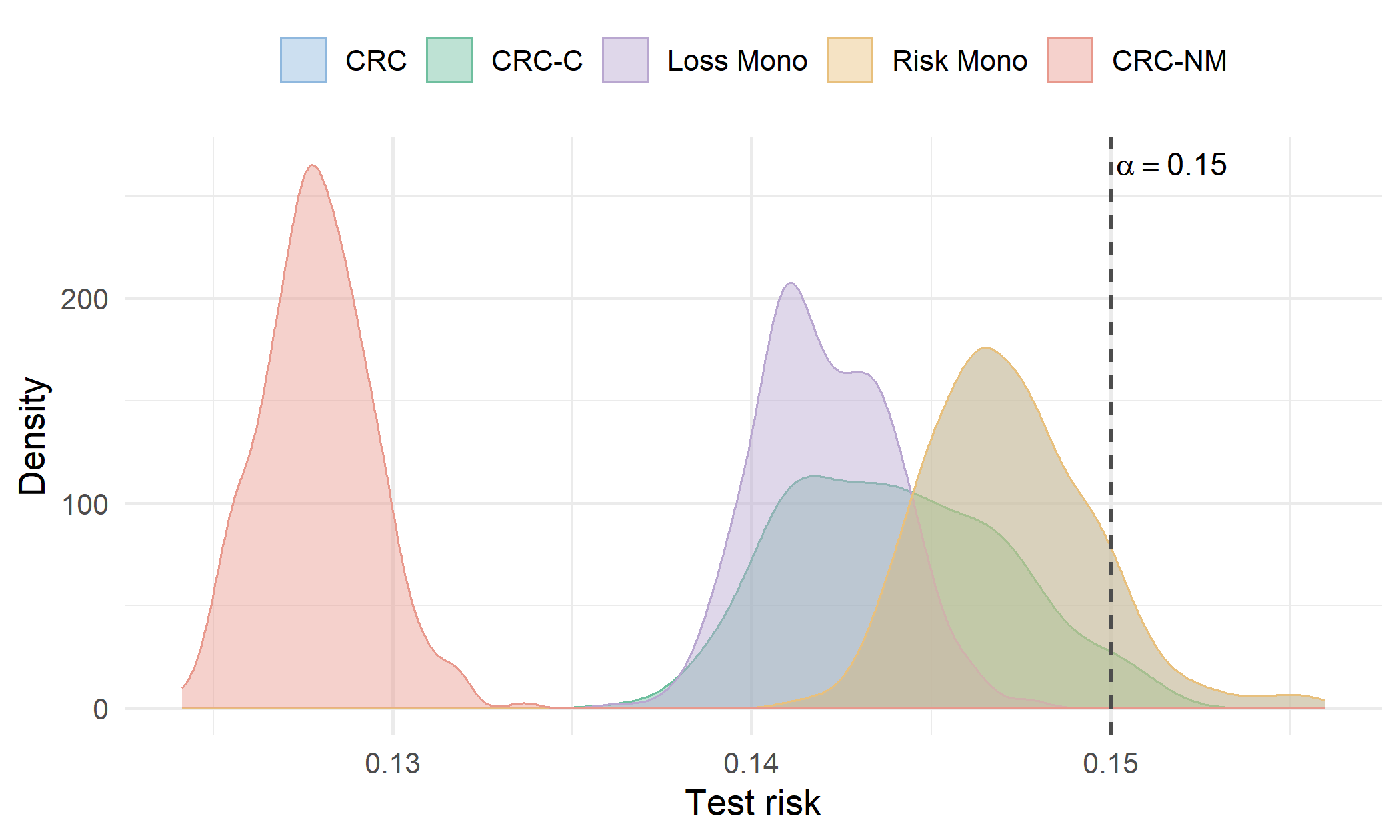}
\includegraphics[height=0.2\textheight,width=0.48\textwidth]
{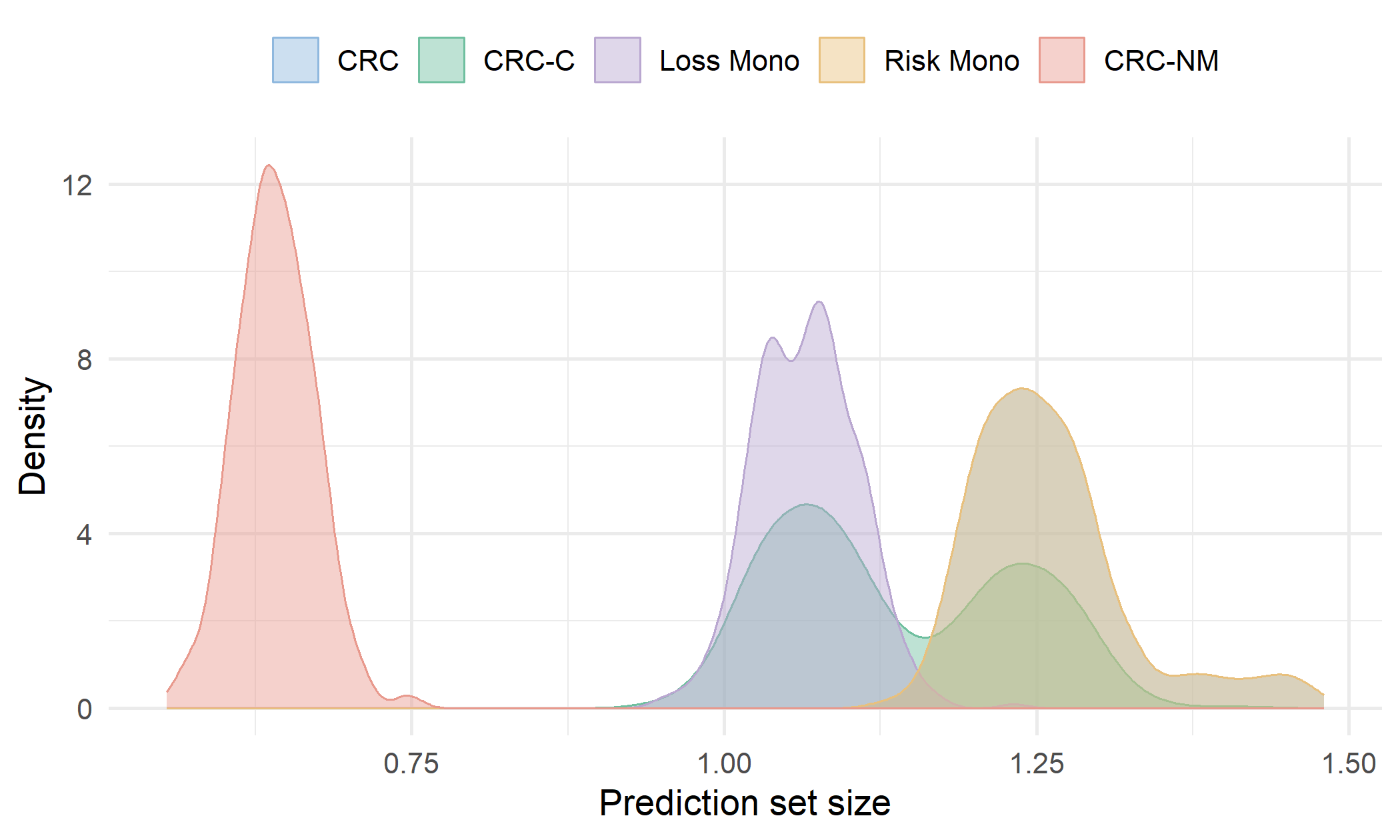}
\caption{Synthetic multilabel experiment.
Left: distributions of empirical test risk across
repetitions, with the dashed line indicating
$\alpha=0.15$.
Right: distributions of prediction-set sizes
across methods.}
\label{fig:synthetic-multilabel}
\end{figure}

Figure~\ref{fig:synthetic-multilabel} compares
the empirical test risks and prediction-set sizes
across methods. Standard CRC and risk
monotonization exceed the target risk level in
approximately $10.5\%$ of repetitions, compared
with approximately $3\%$ for CRC-C.
No exceedances are observed for loss
monotonization or CRC-NM across the $400$
repetitions. In particular, CRC-NM achieves
a mean test risk of $0.128$, below the target
level $\alpha=0.15$.

The prediction-set sizes also differ across
methods. Notably, CRC-NM produces the smallest
prediction sets while maintaining empirical
test risk below the target in all repetitions.
These results illustrate that, under non-monotone
losses, improved empirical risk control does
not necessarily require larger prediction sets.

\subsection{COCO Object Detection}
\label{sec:coco}

We next evaluate the proposed methods on a real-data
object detection task using the COCO 2017 validation
set and a pretrained Faster R-CNN detector.
Prediction sets $C(X;\lambda)$ are constructed by
thresholding detection confidence scores.

We consider the loss
\[
L(X,Y;\lambda)
=
(1-\gamma)
\left(
1-\frac{n_{\mathrm{matched}}}{n_{\mathrm{gt}}}
\right)
+
\gamma\,\phi(|C(X;\lambda)|),
\]
where $n_{\mathrm{matched}}$ is the number of
ground-truth objects matched to selected detections,
$n_{\mathrm{gt}}$ is the number of ground-truth
objects, and
\[
\phi(s)
=
\min\left\{
\frac{(s-K_0)_+}{\tau},\,1
\right\}.
\]
The loss combines a recall-based error with a
penalty for large prediction sets, allowing
non-monotone behavior as the detection threshold
varies. The set-size penalty is zero when the
number of detections does not exceed $K_0$,
increases linearly over the next $\tau$
detections, and is capped at one.

We set $\gamma=0.35$, $K_0=3$, and $\tau=5$.
These constants are specified a priori rather
than tuned on the evaluation data and are chosen
to balance the recall error and the set-size
penalty. The tuning parameter ranges over a
uniform grid of $m=200$ values in $[0.02,0.75]$.
Detections are matched to ground truth using
greedy matching with an intersection-over-union
threshold of $0.3$. The loss is bounded in
$[0,1]$.

We use $n_{\mathrm{cal}}=2{,}500$,
$n_{\mathrm{test}}=800$, and a target risk level
of $\alpha=0.33$. Results are summarized over
$3{,}000$ random splits.

\begin{figure}[t]
\centering
\includegraphics[width=.48\textwidth]
{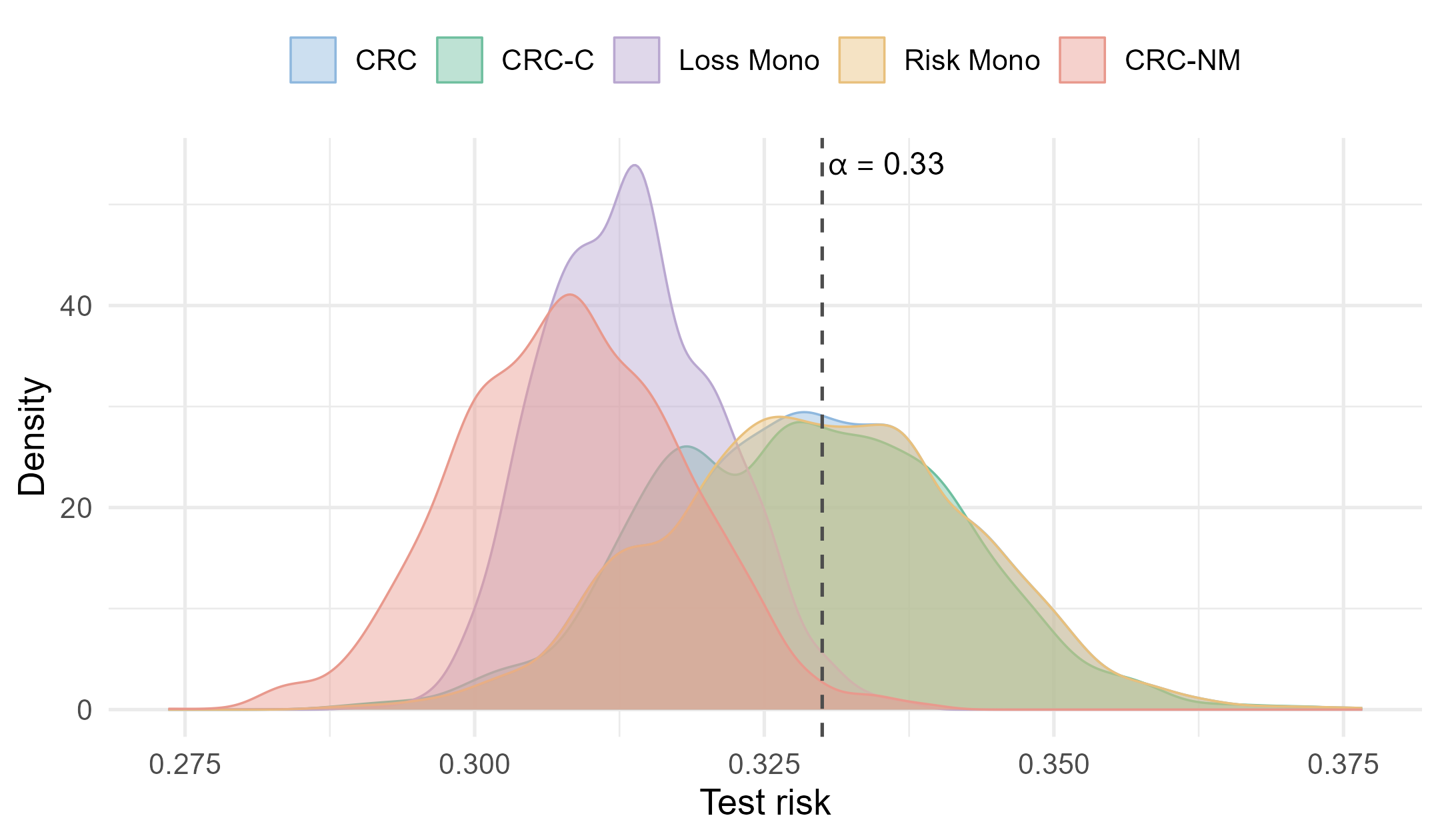}
\includegraphics[width=.48\textwidth]
{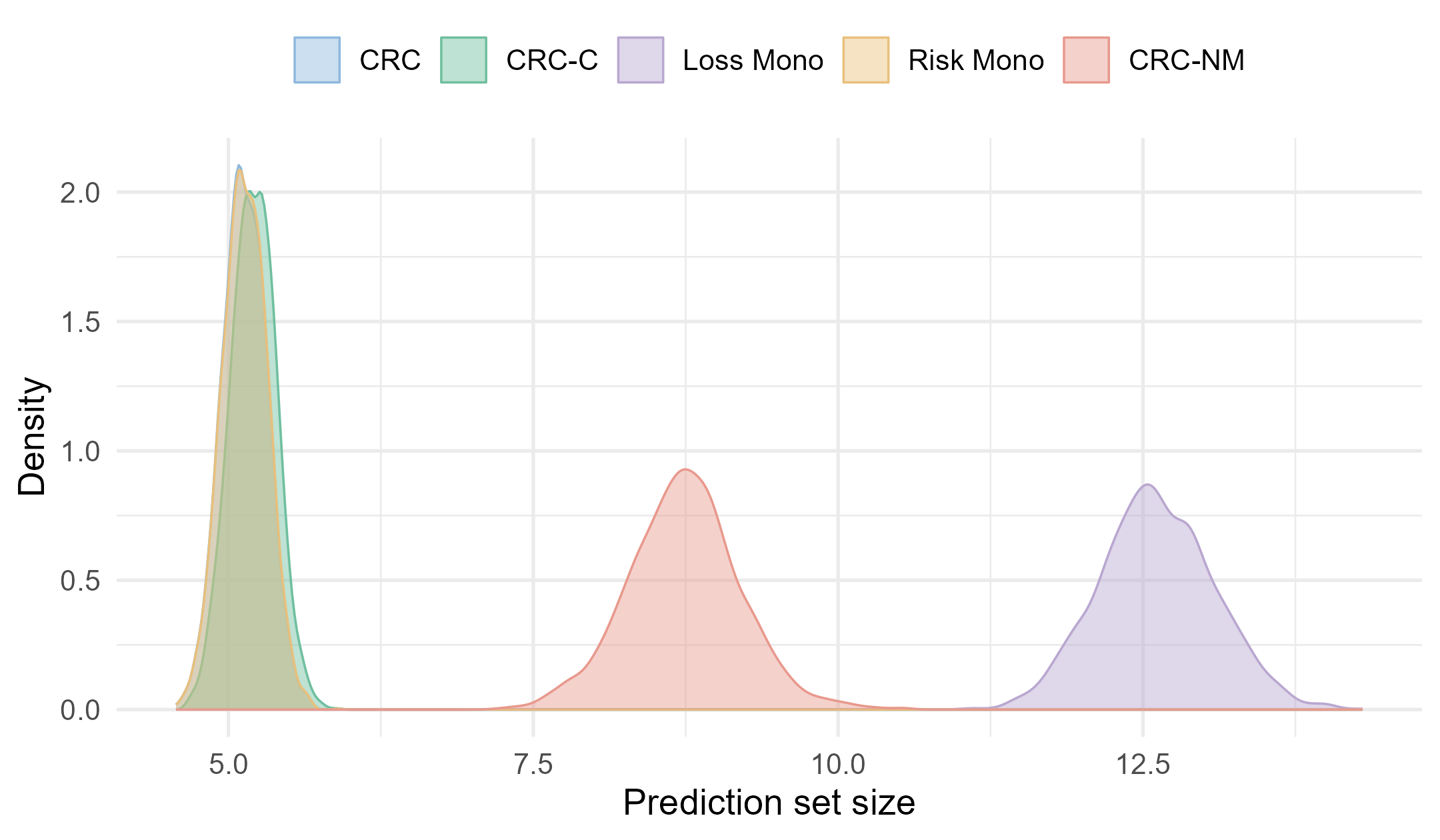}
\caption{COCO object detection experiment.
Left: distributions of empirical test risk
across random splits, with the dashed line
indicating $\alpha=0.33$.
Right: distributions of prediction-set sizes
across methods.}
\label{fig:coco-detection}
\end{figure}

Figure~\ref{fig:coco-detection} compares the
empirical test risks and prediction-set sizes
across methods. Standard CRC and CRC-C achieve
mean test risks close to the target level
$\alpha=0.33$, with empirical test-risk
exceedance frequencies of approximately $48\%$
and $45\%$, respectively. Risk monotonization
exhibits similar behavior, with an exceedance
frequency of approximately $48\%$.

Loss monotonization reduces the exceedance
frequency to approximately $2\%$, but produces
substantially larger prediction sets, with a
mean size of $12.6$ compared with $5.1$ for
standard CRC. In contrast, CRC-NM, using the
adjusted level $\alpha'=\alpha-D(m,n)$,
reduces the exceedance frequency to approximately
$1\%$ while maintaining a moderate mean
prediction-set size of $8.7$. Its mean empirical
test risk is approximately $0.31$, below the
target level.

The empirical terminal risk is approximately
$0.32$, below the target level $\alpha=0.33$.
This is consistent with the population-risk
condition $R(\lambda_m)\le\alpha$ in
Theorem~\ref{thm:nonmonotonic}, although the
empirical estimate alone does not establish
this condition. The stronger requirement
$R(\lambda_m)\le\alpha'$, needed for the exact
expected-risk guarantee of the adjusted
procedure, is not established in this
experiment. Nevertheless, the results provide
empirical evidence of the effectiveness of
CRC-NM under non-monotone losses.

Overall, the synthetic and COCO experiments
illustrate the challenges of risk control
under non-monotone losses. In the COCO experiment,
CRC-NM substantially reduces the frequency
of empirical test-risk exceedances relative
to the unadjusted procedures, while producing
considerably smaller prediction sets than
loss monotonization. These findings demonstrate
the favorable empirical performance of CRC-NM
in balancing risk control and prediction-set
size without modifying the original loss
function.

\section{Extensions to Distributional Shift}
\label{sec:distributional_shift}

We extend the proposed framework to settings with
\emph{distributional shift}, where the calibration
and deployment distributions differ. Such settings
arise when calibration data are collected from one
population or environment, while predictions are
evaluated under another. The objective is therefore
to control the risk under the test distribution
rather than the calibration distribution.

Let $(X_1,Y_1),\ldots,(X_n,Y_n)$ be i.i.d.\
observations from $P_{\mathrm{train}}$, and let
$(X_{n+1},Y_{n+1})$ be drawn independently from
$P_{\mathrm{test}}$. Assume that
$P_{\mathrm{test}}\ll P_{\mathrm{train}}$, with
known likelihood ratio
\[
w(x,y)
=
\frac{dP_{\mathrm{test}}}
{dP_{\mathrm{train}}}(x,y).
\]
We further assume that
\[
0\le w(x,y)\le W<\infty
\]
almost surely under $P_{\mathrm{train}}$.
The boundedness assumption controls the
variability of importance-weighted losses
and permits the use of concentration inequalities.

Under these assumptions, the change-of-measure
identity gives
\[
\mathbb E_{P_{\mathrm{test}}}[g(X,Y)]
=
\mathbb E_{P_{\mathrm{train}}}
[w(X,Y)g(X,Y)]
\]
for every integrable function $g$.
Thus, the test risk can be estimated using
importance-weighted calibration losses,
without requiring exchangeability between
the calibration observations and the test
observation.

\begin{prop}[Risk control under distributional shift]
\label{prop:distributional_shift}

Suppose that
\[
0\le L(X,Y;\lambda)\le B,
\qquad \lambda\in\Lambda,
\]
almost surely under $P_{\mathrm{train}}$,
where
$\Lambda=\{\lambda_1<\cdots<\lambda_m\}$.
Let $0\le\alpha\le B$, and define
\[
R_{\mathrm{test}}(\lambda)
=
\mathbb E_{P_{\mathrm{test}}}
[L(X,Y;\lambda)].
\]

For $i=1,\ldots,n$, define the weighted losses
\[
L_i^{\mathrm w}(\lambda)
=
w(X_i,Y_i)L(X_i,Y_i;\lambda)
\]
and the weighted empirical risk
\[
\hat R_n^{\mathrm w}(\lambda)
=
\frac1n\sum_{i=1}^n
L_i^{\mathrm w}(\lambda).
\]
The weighted losses satisfy
$0\le L_i^{\mathrm w}(\lambda)\le WB$.

Consider the selector
\[
\hat\lambda
=
\inf\left\{
\lambda\in\Lambda:
\frac{n}{n+1}\hat R_n^{\mathrm w}(\lambda)
+\frac{WB}{n+1}
\le\alpha
\right\},
\]
with $\inf\varnothing=\lambda_m$.

Suppose further that there exists a constant
$a>0$ such that
\[
R_{\mathrm{test}}(\lambda_m)
\le
\alpha-\frac{WB-\alpha}{n}-a.
\]
Then
\[
\begin{aligned}
&\mathbb E[
L(X_{n+1},Y_{n+1};\hat\lambda)
]\\
&\quad\le
\alpha+D_W(m,n)
+(B-\alpha)
\exp\left(
-\frac{2na^2}{W^2B^2}
\right),
\end{aligned}
\]
where
\[
D_W(m,n)
=
WB\sqrt{\frac{\log(2m)}{2n}}
+
\frac{WB}{2\sqrt{2n\log(2m)}}.
\]
The expectation is taken jointly over the
calibration sample and the independent
test observation.
\end{prop}

Proposition~\ref{prop:distributional_shift}
extends the finite-sample risk analysis to
distributional shift through importance
weighting. The excess-risk bound contains
two components: a uniform concentration term
$D_W(m,n)$ and an exponentially decreasing
term accounting for the possibility that
the weighted feasible set is empty.
The latter is controlled by the assumed
margin at the terminal parameter.

For fixed $W$, $B$, and $a>0$, the excess-risk
bound converges to zero as $n$ increases,
provided $\log m=o(n)$.

\paragraph{Special case: covariate shift.}
An important special case is covariate shift,
where
\[
P_{\mathrm{test}}(Y\mid X)
=
P_{\mathrm{train}}(Y\mid X).
\]
In this setting, the likelihood ratio
depends only on the covariates:
\[
w(x,y)=w(x)
=
\frac{dP_{\mathrm{test}}^X}
{dP_{\mathrm{train}}^X}(x).
\]
Consequently,
Proposition~\ref{prop:distributional_shift}
applies directly with covariate-based
importance weights, provided its remaining
assumptions are satisfied.

\paragraph{Remark (effect of the weight bound).}
The concentration term $D_W(m,n)$ grows
linearly with $W$, reflecting the increased
variability introduced by importance weighting.
Large likelihood ratios may therefore require
larger calibration samples to obtain informative
risk bounds. Truncating large weights can
reduce variability but generally introduces
bias in the estimation of the target risk.
Extensions to unbounded or heavy-tailed
importance weights are beyond the scope
of this work.

\paragraph{Remark (estimated importance weights).}
Proposition~\ref{prop:distributional_shift}
assumes that the likelihood ratio is known.
In practice, it may be replaced by an estimate
$\hat w$. For any fixed $\lambda$, the
corresponding population weighted risk satisfies
\[
\begin{aligned}
&\left|
\mathbb E_{P_{\mathrm{train}}}
[\hat w(X,Y)L(X,Y;\lambda)]
-
R_{\mathrm{test}}(\lambda)
\right|\\
&\qquad\le
B\|\hat w-w\|_{L^1(P_{\mathrm{train}})}.
\end{aligned}
\]
Thus, estimation error in the likelihood ratio
introduces an additional approximation error.
Establishing a finite-sample guarantee for
the resulting data-dependent selector requires
further control of weight-estimation and
selection errors. Such extensions, potentially
using sample splitting or cross-fitting, are
left for future work.

\section{Conclusion}
\label{sec:conc}

We extend conformal risk control to bounded, non-monotone loss
functions over discrete parameter grids. Our results show that
monotonicity is not essential for finite-sample expectation control:
the excess risk induced by parameter selection grows only
logarithmically with the grid size and diminishes as the calibration
sample size increases. This provides a simple and general framework
for controlling predictive risk even when the loss function does not
satisfy the structural assumptions typically imposed by existing CRC
methods.

From a conceptual perspective, our analysis interprets parameter
selection as a statistical model selection problem over a finite grid.
Rather than enforcing monotonicity or relying on algorithmic stability
conditions, we directly control the excess risk arising from searching
over candidate parameters. This viewpoint leads to explicit
finite-sample guarantees that apply to arbitrary bounded loss
functions without requiring monotonicity. We further establish a
matching lower bound showing that the excess-risk rate is sharp for
the CRC selection rule under the same population-safe endpoint
condition. This demonstrates that the additional excess-risk term
cannot generally be eliminated without further structural
assumptions.

These results broaden the applicability of conformal risk control to
settings where loss functions are inherently non-monotone. Such
situations commonly arise when a tuning parameter governs a discrete
decision rule or classification threshold, where small changes in the
parameter can lead to abrupt changes in predictions. Several directions
for future work remain open, including extensions to continuous
parameter spaces, handling heavy-tailed or unbounded loss functions,
and obtaining sharper bounds under additional structural assumptions.

Overall, our findings demonstrate that reliable conformal risk control
can be achieved without monotonicity assumptions, requiring only
bounded losses and a finite parameter grid.


\section*{Supplementary Materials}

The online supplementary material contains additional technical proofs,
auxiliary lemmas, and supporting numerical results.

\bibhang=1.7pc
\bibsep=2pt
\fontsize{9}{14pt plus.8pt minus .6pt}\selectfont
\renewcommand\bibname{\large \bf References}

\expandafter\ifx\csname natexlab\endcsname\relax\def\natexlab#1{#1}\fi
\expandafter\ifx\csname url\endcsname\relax
\def\url#1{\texttt{#1}}\fi
\expandafter\ifx\csname urlprefix\endcsname\relax\def\urlprefix{URL}\fi

\bibliographystyle{chicago}
\bibliography{References}

\vskip .65cm
\noindent
Department of Mathematical Sciences, New Jersey Institute of Technology, Newark, NJ, USA.\\
E-mail: (ta429@njit.edu; yl238@njit.edu; wenge.guo@njit.edu)

\newpage
\renewcommand{\baselinestretch}{2}

\markright{ \hbox{\footnotesize\rm }\hfill\\[-13pt]
\hbox{\footnotesize\rm}\hfill }

\markboth{\hfill{\footnotesize\rm Tareq Aldirawi, Yun Li AND Wenge Guo} \hfill}
{\hfill {\footnotesize\rm CRC under Non-Monotone Losses} \hfill}

\renewcommand{\thefootnote}{}
$\ $\par

\setcounter{section}{0}
\setcounter{equation}{0}
\setcounter{figure}{0}
\setcounter{table}{0}
\setcounter{theorem}{0}
\setcounter{lemma}{0}
\setcounter{corollary}{0}
\setcounter{proposition}{0}
\setcounter{remark}{0}
\setcounter{definition}{0}
\setcounter{example}{0}

\fontsize{12}{14pt plus.8pt minus .6pt}\selectfont

\centerline{\large\bf Conformal Risk Control under Non-Monotone Losses:}
\vspace{2pt}
\centerline{\large\bf Theory and Finite-Sample Guarantees}
\vspace{.25cm}
\centerline{Tareq Aldirawi, Yun Li and Wenge Guo}
\vspace{.4cm}
\centerline{\it New Jersey Institute of Technology}
\vspace{.55cm}
\centerline{\bf Supplementary Material}
\vspace{.55cm}

\fontsize{9}{11.5pt plus.8pt minus .6pt}\selectfont

\noindent
This supplementary material provides detailed proofs and additional supporting results, including auxiliary lemmas, refined bounds, and numerical comparisons. It also contains detailed derivations for Section~3.1, extended methodological comparisons for Section~5, and full experimental descriptions for Section~6.

\paragraph{Organization of the Supplementary Material.}
Section~A presents detailed derivations for the counterexample in Section~3.1.
Sections~B and~C provide auxiliary lemmas and proofs of the main results.
Section~D contains extended methodological comparisons related to Section~5,
and Section~E provides complete experimental details for Section~6.
Section~F presents variance-sensitive refinements and additional numerical results.
\par

\vspace{10pt}

\fontsize{12}{14pt plus.8pt minus .6pt}\selectfont

\renewcommand{\thesection}{S\arabic{section}}
\renewcommand{\theequation}{S\arabic{section}.\arabic{equation}}
\renewcommand{\thefigure}{S\arabic{figure}}
\renewcommand{\thetable}{S\arabic{table}}

\renewcommand{\thetheorem}{S\arabic{theorem}}
\renewcommand{\thelemma}{S\arabic{lemma}}
\renewcommand{\thecorollary}{S\arabic{corollary}}
\renewcommand{\theproposition}{S\arabic{proposition}}
\renewcommand{\theremark}{S\arabic{remark}}
\renewcommand{\thedefinition}{S\arabic{definition}}
\renewcommand{\theexample}{S\arabic{example}}
\begin{appendix}
\section{Supplementary Material for Section~\ref{sec:counterexample}}
\label{sec:supp-counterexample}

\subsection{Detailed analysis of the counterexample}

We provide the full derivation underlying the scaling behavior described in Section~\ref{sec:counterexample}.

Recall that
\[
\mathbb{E}[L_{n+1}(\hat{\lambda})]
=
p\bigl(1-(1-q)^m\bigr),
\]
where
\[
q=\mathbb{P}\!\left(\frac1n\sum_{i=1}^n Z_i
\le \alpha-\frac1{n+1}\right),
\quad Z_i\sim\mathrm{Bern}(p).
\]

\subsubsection*{Bounding $(1-q)^m$}

Using the elementary bounds
\[
1 - mq \le (1-q)^m \le e^{-mq},
\]
we obtain
\begin{equation}
p\bigl(1 - e^{-mq}\bigr)
\;\le\;
\mathbb{E}[L_{n+1}(\hat{\lambda})]
\;\le\;
pmq.
\label{eq:supp-basic-bounds}
\end{equation}

\subsubsection*{Bounding $q$}

Let
\[
S_n = \sum_{i=1}^n Z_i \sim \mathrm{Bin}(n,p),
\quad
t = \alpha - \frac{1}{n+1},
\]
so that
\[
q = \mathbb{P}\!\left( \frac{S_n}{n} \le t \right).
\]

\paragraph{Upper bound.}
Since $0 < t < p$, Hoeffding's inequality gives
\begin{equation}
q \le \exp\!\big(-2n(p - t)^2\big).
\label{eq:supp-q-upper}
\end{equation}

\paragraph{Lower bound.}
A simple lower bound is obtained by
\begin{equation}
q \ge \mathbb{P}(S_n = 0) = (1-p)^n.
\label{eq:supp-q-lower}
\end{equation}

\subsubsection*{Combined bounds}

Substituting \eqref{eq:supp-q-upper} and \eqref{eq:supp-q-lower} into \eqref{eq:supp-basic-bounds} yields
\begin{equation}
p\Bigl(1 - e^{-m(1-p)^n}\Bigr)
\;\le\;
\mathbb{E}[L_{n+1}(\hat{\lambda})]
\;\le\;
pm\,\exp\!\big(-2n(p - t)^2\big).
\label{eq:supp-final-bounds}
\end{equation}

\subsubsection*{Scaling regimes}

The bounds in \eqref{eq:supp-final-bounds} reveal two distinct regimes:

\paragraph{Controlled regime.}
If
\[
pm\,\exp\!\big(-2n(p - t)^2\big) \le \alpha,
\]
then $\mathbb{E}[L_{n+1}(\hat{\lambda})] \le \alpha$. This holds, for example, when
\[
n \gtrsim \frac{\log m}{2(p-\alpha)^2}.
\]

\paragraph{Uncontrolled regime.}
If
\[
p\Bigl(1 - e^{-m(1-p)^n}\Bigr) > \alpha,
\]
then $\mathbb{E}[L_{n+1}(\hat{\lambda})] > \alpha$. This occurs when
\[
m(1-p)^n \gtrsim 1,
\quad \text{equivalently} \quad
n \lesssim \frac{\log m}{-\log(1-p)}.
\]

\medskip
\noindent
These results formally justify the scaling transition described in the main text: 
risk control depends on the relative growth of the sample size $n$ and the grid size $m$.
\section{Lemmas}
\label{appendix:A}
\begin{lemma}[Uniform Concentration Bound]\label{lemma:uniform}
Let $L_1(\lambda), \ldots, L_n(\lambda)$ be i.i.d.\ random variables with 
$L_i(\lambda) \in [0, B]$ for all $\lambda \in \Lambda = \{\lambda_1, \ldots, \lambda_m\}$. 
Define the empirical risk $\hat{R}_n(\lambda) = \frac{1}{n}\sum_{i=1}^{n} L_i(\lambda)$ 
and the true risk $R(\lambda) = \mathbb{E}[L_i(\lambda)]$. Then
\[
\mathbb{E}\left[\sup_{\lambda \in \Lambda} \left|\hat{R}_n(\lambda) - R(\lambda)\right|\right] 
\leq B\sqrt{\frac{\log(2m)}{2n}} + \frac{B}{2\sqrt{2n\log(2m)}}.
\]
\end{lemma}

\begin{proof}
Define $Z = \sup_{\lambda \in \Lambda} |\hat{R}_n(\lambda) - R(\lambda)|$. By Hoeffding's inequality, for any fixed $\lambda \in \Lambda$ and $t > 0$,
\[
\mathbb{P}\left(|\hat{R}_n(\lambda) - R(\lambda)| \geq t\right) 
\leq 2\exp\left(-\frac{2nt^2}{B^2}\right).
\]
Applying the union bound over $\Lambda$,
\begin{equation}\label{eq:tail_bound}
\mathbb{P}(Z \geq t) \leq 2m \exp\left(-\frac{2nt^2}{B^2}\right).
\end{equation}

Since $Z \leq B$ almost surely, we write
\[
\mathbb{E}[Z] = \int_0^{B} \mathbb{P}(Z \geq t)\, dt.
\]
Define $t^* = B\sqrt{\frac{\log(2m)}{2n}}$, chosen so that the tail bound \eqref{eq:tail_bound} 
equals 1 at $t = t^*$:
\[
2m \exp\left(-\frac{2n(t^*)^2}{B^2}\right) = 2m \exp(-\log(2m)) = 1.
\]
We split the integral at $t^*$:
\[
\mathbb{E}[Z] = \int_0^{t^*} \mathbb{P}(Z \geq t)\, dt + \int_{t^*}^{B} \mathbb{P}(Z \geq t)\, dt.
\]

For the first integral, $\mathbb{P}(Z \geq t) \leq 1$, so
\[
\int_0^{t^*} \mathbb{P}(Z \geq t)\, dt \leq t^* = B\sqrt{\frac{\log(2m)}{2n}}.
\]

For the second integral, we apply \eqref{eq:tail_bound}:
\[
\int_{t^*}^{B} \mathbb{P}(Z \geq t)\, dt \leq \int_{t^*}^{\infty} 2m \exp\left(-\frac{2nt^2}{B^2}\right) dt.
\]
Substituting $u = t\sqrt{2n}/B$ (so $dt = B\, du/\sqrt{2n}$) with $u^* = t^*\sqrt{2n}/B = \sqrt{\log(2m)}$,
\[
\int_{t^*}^{\infty} 2m \exp\left(-\frac{2nt^2}{B^2}\right) dt 
= \frac{2mB}{\sqrt{2n}} \int_{u^*}^{\infty} e^{-u^2}\, du.
\]

Using the Gaussian tail bound $\int_v^{\infty} e^{-u^2}\, du \leq \frac{e^{-v^2}}{2v}$ for $v > 0$,
\[
\int_{u^*}^{\infty} e^{-u^2}\, du \leq \frac{e^{-(u^*)^2}}{2u^*} 
= \frac{e^{-\log(2m)}}{2\sqrt{\log(2m)}} = \frac{1}{4m\sqrt{\log(2m)}}.
\]

Therefore,
\[
\int_{t^*}^{B} \mathbb{P}(Z \geq t)\, dt \leq \frac{2mB}{\sqrt{2n}} \cdot \frac{1}{4m\sqrt{\log(2m)}} 
= \frac{B}{2\sqrt{2n\log(2m)}}.
\]

Combining both parts:
\[
\mathbb{E}[Z] \leq B\sqrt{\frac{\log(2m)}{2n}} + \frac{B}{2\sqrt{2n\log(2m)}}.
\]

\end{proof}

\section{Proofs}
\label{appendix:B}

\subsection{Proof of Theorem~\ref{thm:nonmonotonic}}

\begin{proof}
Let
\[
R(\lambda)=\mathbb E[L_{n+1}(\lambda)]
\]
and define
\[
\Delta_n
=
\max_{\lambda\in\Lambda}
\left|
\hat R_n(\lambda)-R(\lambda)
\right|.
\]
Since the loss vectors are i.i.d. and
$0\le L_i(\lambda)\le B$ almost surely,
Lemma~\ref{lemma:uniform} gives
\[
\mathbb E[\Delta_n]\le D(m,n),
\]
where
\[
D(m,n)
=
B\sqrt{\frac{\log(2m)}{2n}}
+
\frac{B}{2\sqrt{2n\log(2m)}}.
\]

We consider two cases, depending on whether
the feasible set defining $\hat\lambda$
is empty.

\medskip
\noindent\textbf{Case 1: The feasible set is nonempty.}
By the definition of $\hat\lambda$,
\[
\frac{n}{n+1}\hat R_n(\hat\lambda)
+
\frac{B}{n+1}
\le\alpha.
\]
Since $0\le\alpha\le B$, it follows that
\[
\hat R_n(\hat\lambda)
\le
\alpha-\frac{B-\alpha}{n}
\le\alpha.
\]
Consequently,
\[
\begin{aligned}
R(\hat\lambda)
&\le
\hat R_n(\hat\lambda)+\Delta_n\\
&\le
\alpha+\Delta_n.
\end{aligned}
\]

\medskip
\noindent\textbf{Case 2: The feasible set is empty.}
By convention, $\hat\lambda=\lambda_m$.
The population-risk condition
$R(\lambda_m)\le\alpha$ therefore gives
\[
R(\hat\lambda)
=
R(\lambda_m)
\le\alpha.
\]
Hence, in this case as well,
\[
R(\hat\lambda)\le\alpha+\Delta_n.
\]

Combining the two cases yields
\[
R(\hat\lambda)\le\alpha+\Delta_n
\qquad\text{almost surely}.
\]

Since $\hat\lambda$ is measurable with respect
to the calibration loss vectors
and the test loss vector is independent of
the calibration data,
\[
\mathbb E[L_{n+1}(\hat\lambda)]
=
\mathbb E[R(\hat\lambda)].
\]
Taking expectations and applying
Lemma~\ref{lemma:uniform}, we obtain
\[
\begin{aligned}
\mathbb E[L_{n+1}(\hat\lambda)]
&=
\mathbb E[R(\hat\lambda)]\\
&\le
\alpha+\mathbb E[\Delta_n]\\
&\le
\alpha+D(m,n).
\end{aligned}
\]
This establishes the first assertion.

\medskip
\noindent\textbf{Adjusted procedure.}
Let $\alpha'=\alpha-D(m,n)>0$ and suppose
that $R(\lambda_m)\le\alpha'$.
The adjusted selector
$\hat\lambda^{\mathrm{adj}}$
is defined by replacing $\alpha$ with
$\alpha'$ in the original selection rule.
Since $0<\alpha'\le B$, the preceding
two-case argument applies with $\alpha'$
in place of $\alpha$, while the same
uniform deviation bound remains valid.
Therefore,
\[
\begin{aligned}
\mathbb E[
L_{n+1}(\hat\lambda^{\mathrm{adj}})
]
&\le
\alpha'+D(m,n)\\
&=\alpha.
\end{aligned}
\]
This completes the proof.
\end{proof}

\subsection{Proof of Proposition~\ref{prop:sharpness}}

\begin{proof}
Fix $\alpha\in(0,1)$ and let
$\Lambda=\{\lambda_1,\ldots,\lambda_m\}$,
where $\lambda_1<\cdots<\lambda_m$.
For $0<\delta\le(1-\alpha)/2$,
consider independent binary loss coordinates
satisfying
\[
L_i(\lambda_k)\sim
\begin{cases}
\operatorname{Bernoulli}(\alpha+\delta),
& k=1,\ldots,m-1,\\
\operatorname{Bernoulli}(\alpha),
& k=m.
\end{cases}
\]
The loss vectors are i.i.d.\ across
$i=1,\ldots,n+1$. Thus,
\[
R(\lambda_k)=\alpha+\delta
\quad(k<m),
\qquad
R(\lambda_m)=\alpha,
\]
so the safe-endpoint assumption holds.

For each $k<m$, the CRC feasibility
condition is equivalent to
\[
\sum_{i=1}^n L_i(\lambda_k)
\le(n+1)\alpha-1.
\]
Let
\[
q_{n,\delta}
=
\mathbb P\left\{
\operatorname{Bin}(n,\alpha+\delta)
\le(n+1)\alpha-1
\right\}.
\]
By independence across coordinates,
the probability that at least one of
the first $m-1$ parameters is feasible
equals $1-(1-q_{n,\delta})^{m-1}$.
Consequently,
\begin{equation}
\label{eq:sharpness-identity}
\mathbb E[L_{n+1}(\hat\lambda)]
=
\alpha+
\delta\left[1-(1-q_{n,\delta})^{m-1}\right].
\end{equation}

We next lower-bound $q_{n,\delta}$.
Let $Q$ and $P_\delta$ denote the
distributions of $n$ independent
Bernoulli observations with means
$\alpha$ and $\alpha+\delta$,
respectively, and define
\[
A=
\left\{
\sum_{i=1}^n X_i
\le(n+1)\alpha-1
\right\}.
\]
Since
\[
\frac{(n+1)\alpha-1-n\alpha}
{\sqrt{n\alpha(1-\alpha)}}
\longrightarrow0,
\]
the central limit theorem gives
$Q(A)\to1/2$. Hence there exists
$N_\alpha\ge1$ such that
\[
Q(A)\ge\frac14,
\qquad n\ge N_\alpha.
\]

By Cauchy--Schwarz,
\[
\begin{aligned}
Q(A)^2
&=
\left[
\mathbb E_{P_\delta}
\left\{
\frac{dQ}{dP_\delta}
\mathbf 1_A
\right\}
\right]^2\\
&\le
P_\delta(A)
\mathbb E_{P_\delta}
\left[
\left(\frac{dQ}{dP_\delta}\right)^2
\right].
\end{aligned}
\]
For Bernoulli observations,
\[
\begin{aligned}
\mathbb E_{P_\delta}
\left[
\left(\frac{dQ}{dP_\delta}\right)^2
\right]
&=
\left(
1+
\frac{\delta^2}
{(\alpha+\delta)(1-\alpha-\delta)}
\right)^n\\
&\le
\exp\left\{
\frac{2n\delta^2}{\alpha(1-\alpha)}
\right\},
\end{aligned}
\]
where the inequality follows from
$\delta\le(1-\alpha)/2$.
Therefore, for $n\ge N_\alpha$,
\begin{equation}
\label{eq:binomial-lower}
q_{n,\delta}
=P_\delta(A)
\ge
\frac1{16}
\exp\left\{
-\frac{2n\delta^2}{\alpha(1-\alpha)}
\right\}.
\end{equation}

Choose $a_\alpha>0$ such that
\[
a_\alpha
\le
\min\left\{
\frac{1-\alpha}{2},
\frac{\sqrt{\alpha(1-\alpha)}}{2}
\right\},
\]
and set
\[
\delta=
a_\alpha
\min\left\{
1,\sqrt{\frac{\log m}{n}}
\right\}.
\]
Then
\[
\frac{2n\delta^2}{\alpha(1-\alpha)}
\le\frac12\log m,
\]
so \eqref{eq:binomial-lower} yields
\[
q_{n,\delta}\ge\frac1{16\sqrt m}.
\]
For $m\ge4$, it follows that
\[
\begin{aligned}
1-(1-q_{n,\delta})^{m-1}
&\ge1-\exp\{-(m-1)q_{n,\delta}\}\\
&\ge1-\exp\left\{
-\frac{m-1}{16\sqrt m}
\right\}\\
&\ge1-e^{-3/32}.
\end{aligned}
\]

Substituting into
\eqref{eq:sharpness-identity},
we obtain
\[
\mathbb E[L_{n+1}(\hat\lambda)]
\ge
\alpha+
c_\alpha
\min\left\{
1,\sqrt{\frac{\log m}{n}}
\right\},
\]
where
\[
c_\alpha
=
a_\alpha(1-e^{-3/32})>0.
\]
This completes the proof.
\end{proof}

\subsection{Proof of Proposition~\ref{thm:lipschitz}}

\begin{proof}
Let
\[
R(\lambda)
=
\mathbb E[L_{n+1}(\lambda)],
\qquad \lambda\in\mathcal I=[a,b],
\]
and define
\[
\Delta_n
=
\sup_{\lambda\in\mathcal I}
\left|
\hat R_n(\lambda)-R(\lambda)
\right|,
\]
where
\[
\hat R_n(\lambda)
=
\frac{1}{n}\sum_{i=1}^n L_i(\lambda).
\]

We first establish an upper bound on
$\mathbb E[\Delta_n]$.

\medskip
\noindent
\textbf{Step 1: Approximation by an auxiliary grid.}

For any integer $N\ge2$, introduce an auxiliary
grid of $N$ equally spaced points,
\[
\mathcal G_N=\{t_1,\ldots,t_N\},
\]
where
\[
t_j
=
a+\frac{j-1}{N-1}(b-a),
\qquad j=1,\ldots,N.
\]
The grid spacing is
\[
h_N=\frac{b-a}{N-1}.
\]

This auxiliary grid is used only for the
theoretical analysis and does not affect
the selection of $\hat\lambda$.

For every $\lambda\in\mathcal I$, there exists
$t_j\in\mathcal G_N$ such that
\[
|\lambda-t_j|\le\frac{h_N}{2}.
\]

By the uniform Lipschitz assumption,
both $\hat R_n(\lambda)$ and $R(\lambda)$
are $K$-Lipschitz continuous on $\mathcal I$,
almost surely. Consequently,
\[
\begin{aligned}
|\hat R_n(\lambda)-R(\lambda)|
&\le |\hat R_n(t_j)-R(t_j)|\\
&\quad +|\hat R_n(\lambda)-\hat R_n(t_j)|
      +|R(\lambda)-R(t_j)|\\
&\le |\hat R_n(t_j)-R(t_j)|+Kh_N.
\end{aligned}
\]

Taking the supremum over $\lambda\in\mathcal I$
gives
\[
\Delta_n
\le
\max_{1\le j\le N}
|\hat R_n(t_j)-R(t_j)|
+
\frac{K(b-a)}{N-1}.
\]

\medskip
\noindent
\textbf{Step 2: Uniform concentration on the auxiliary grid.}

For each fixed $t_j\in\mathcal G_N$, the random
variables $L_1(t_j),\ldots,L_n(t_j)$
are i.i.d. and take values in $[0,B]$.

By Hoeffding's inequality, for every $t>0$,
\[
\mathbb P\left(
|\hat R_n(t_j)-R(t_j)|>t
\right)
\le
2\exp\left(-\frac{2nt^2}{B^2}\right).
\]

Applying the union bound over the $N$
auxiliary grid points yields
\[
\begin{aligned}
&\mathbb P\left(
\max_{1\le j\le N}
|\hat R_n(t_j)-R(t_j)|>t
\right)\\
&\qquad\le
2N\exp\left(-\frac{2nt^2}{B^2}\right).
\end{aligned}
\]

Let
\[
M_N
=
\max_{1\le j\le N}
|\hat R_n(t_j)-R(t_j)|
\]
and choose
\[
t_0
=
B\sqrt{\frac{\log(2N)}{2n}}.
\]

Using the tail integration identity,
\[
\mathbb E[M_N]
=
\int_0^\infty
\mathbb P(M_N>t)\,dt,
\]
we obtain
\[
\begin{aligned}
\mathbb E[M_N]
&\le
t_0+
\int_{t_0}^\infty
2N\exp\left(-\frac{2nt^2}{B^2}\right)\,dt.
\end{aligned}
\]

For $c>0$ and $u>0$, the Gaussian tail
inequality gives
\[
\int_u^\infty e^{-ct^2}\,dt
\le
\frac{e^{-cu^2}}{2cu}.
\]

Taking $c=2n/B^2$ and $u=t_0$,
\[
\begin{aligned}
&\int_{t_0}^\infty
2N\exp\left(-\frac{2nt^2}{B^2}\right)\,dt\\
&\quad\le
\frac{2NB^2}{4nt_0}
\exp\left(-\frac{2nt_0^2}{B^2}\right)\\
&\quad=
\frac{B^2}{4nt_0},
\end{aligned}
\]
where the last equality follows from
\[
2N\exp\left(-\frac{2nt_0^2}{B^2}\right)=1.
\]

Therefore,
\[
\begin{aligned}
\mathbb E[M_N]
&\le
t_0+\frac{B^2}{4nt_0}\\
&=
B\sqrt{\frac{\log(2N)}{2n}}
+
\frac{B}{2\sqrt{2n\log(2N)}}\\
&=
D(N,n).
\end{aligned}
\]

Combining this with the auxiliary-grid
approximation yields
\[
\mathbb E[\Delta_n]
\le
D(N,n)+\frac{K(b-a)}{N-1}.
\]

Since this inequality holds for every
integer $N\ge2$, we obtain
\[
\mathbb E[\Delta_n]
\le
\widetilde D(n,K,b-a).
\]

\medskip
\noindent
\textbf{Step 3: Risk control for the selected parameter.}

Define the feasible set
\[
\mathcal F
=
\left\{
\lambda\in\Lambda:
\frac{n}{n+1}\hat R_n(\lambda)
+\frac{B}{n+1}\le\alpha
\right\}.
\]

We consider two cases.

If $\mathcal F\ne\varnothing$, then
$\hat\lambda\in\mathcal F$ and
\[
\frac{n}{n+1}\hat R_n(\hat\lambda)
+\frac{B}{n+1}
\le\alpha.
\]
Since $0\le\alpha\le B$, this implies
\[
\begin{aligned}
\hat R_n(\hat\lambda)
&\le
\alpha-\frac{B-\alpha}{n}\\
&\le\alpha.
\end{aligned}
\]
Consequently,
\[
\begin{aligned}
R(\hat\lambda)
&\le
\hat R_n(\hat\lambda)+\Delta_n\\
&\le
\alpha+\Delta_n.
\end{aligned}
\]

If $\mathcal F=\varnothing$, then
$\hat\lambda=\lambda_m$ by convention.
The population-risk condition in
Theorem~\ref{thm:nonmonotonic} gives
\[
R(\hat\lambda)
=
R(\lambda_m)
\le\alpha.
\]

Thus, in both cases,
\[
R(\hat\lambda)
\le
\alpha+\Delta_n
\qquad\text{almost surely}.
\]

\medskip
\noindent
\textbf{Step 4: Taking expectations.}

Since $\hat\lambda$ depends only on
the calibration observations and is
independent of the test observation,
\[
\begin{aligned}
\mathbb E[L_{n+1}(\hat\lambda)]
&=
\mathbb E\left[
\mathbb E\left[
L_{n+1}(\hat\lambda)
\mid Z_1,\ldots,Z_n
\right]
\right]\\
&=
\mathbb E[R(\hat\lambda)].
\end{aligned}
\]

Therefore,
\[
\begin{aligned}
\mathbb E[L_{n+1}(\hat\lambda)]
&\le
\alpha+\mathbb E[\Delta_n]\\
&\le
\alpha+\widetilde D(n,K,b-a).
\end{aligned}
\]

Furthermore, Theorem~\ref{thm:nonmonotonic}
gives
\[
\mathbb E[L_{n+1}(\hat\lambda)]
\le
\alpha+D(m,n).
\]

Combining the two bounds yields
\[
\boxed{
\mathbb E[L_{n+1}(\hat\lambda)]
\le
\alpha+
\min\left\{
D(m,n),\,
\widetilde D(n,K,b-a)
\right\}.
}
\]

This completes the proof.
\end{proof}

\subsection{Proof of Proposition~\ref{prop:distributional_shift}}

\begin{proof}
Let
\[
\Delta_n^{\mathrm w}
=
\max_{\lambda\in\Lambda}
\left|
\hat R_n^{\mathrm w}(\lambda)
-
R_{\mathrm{test}}(\lambda)
\right|.
\]
By the change-of-measure identity and
Hoeffding's inequality,
\[
\mathbb E[\Delta_n^{\mathrm w}]
\le D_W(m,n).
\]

Let $A$ denote the event that the weighted
feasible set is nonempty. On $A$, the
selection criterion implies
$\hat R_n^{\mathrm w}(\hat\lambda)\le\alpha$,
and hence
\[
R_{\mathrm{test}}(\hat\lambda)
\le\alpha+\Delta_n^{\mathrm w}.
\]

On $A^c$, the selector equals $\lambda_m$,
so $R_{\mathrm{test}}(\hat\lambda)\le B$.
Moreover, the terminal-risk margin implies
\[
\begin{aligned}
\mathbb P(A^c)
&\le
\mathbb P\left(
\hat R_n^{\mathrm w}(\lambda_m)
>
\alpha-\frac{WB-\alpha}{n}
\right)\\
&\le
\exp\left(
-\frac{2na^2}{W^2B^2}
\right),
\end{aligned}
\]
where the last inequality follows from
Hoeffding's inequality.

Combining the two cases yields
\[
\begin{aligned}
R_{\mathrm{test}}(\hat\lambda)
&\le
\alpha+\Delta_n^{\mathrm w}
+(B-\alpha)\mathbf 1_{A^c}.
\end{aligned}
\]
Since the test observation is independent
of the calibration sample,
\[
\mathbb E[
L(X_{n+1},Y_{n+1};\hat\lambda)
]
=
\mathbb E[
R_{\mathrm{test}}(\hat\lambda)
].
\]
Taking expectations establishes the result.
\end{proof}

\section{Additional Details for Section~\ref{sec:existing-comparison}}
\label{sec:supp-existing-comparison}

This section provides detailed formulations, theoretical comparisons,
and additional empirical illustrations for the methods discussed in
Section~\ref{sec:existing-comparison}.

\subsection{Monotonization Techniques}
\label{sec:supp-monotonization}

\cite{angelopoulos2024conformal} proposed two
approaches that restore the monotonicity required by conformal risk
control (CRC) by modifying either the loss function or the empirical
risk.

\paragraph{Loss monotonization.}
Given calibration losses $\{L_i(\lambda)\}_{i=1}^n$, define
\[
  \tilde{L}_i(\lambda) := \sup_{t \ge \lambda} L_i(t),
  \qquad
  \tilde{R}_n(\lambda) := \frac{1}{n}\sum_{i=1}^n \tilde{L}_i(\lambda).
\]
Then $\tilde{L}_i(\lambda)$ is non-increasing in $\lambda$ and satisfies
$\tilde{L}_i(\lambda) \ge L_i(\lambda)$.
Applying CRC to $\tilde{R}_n$ yields the threshold
\[
  \tilde{\lambda}
  = \inf\left\{\lambda :
  \frac{n}{n+1}\tilde{R}_n(\lambda) + \frac{B}{n+1} \le \alpha
  \right\},
\]
which guarantees $\mathbb{E}[L_{n+1}(\tilde{\lambda})] \le \alpha$.

\paragraph{Risk monotonization.}
Alternatively, monotonicity can be imposed on the empirical risk:
\[
  \hat{R}_n^{\uparrow}(\lambda)
  := \sup_{t \ge \lambda} \hat{R}_n(t),
  \qquad
  \hat{\lambda}_n^{\uparrow}
  := \inf\bigl\{\lambda :
  \hat{R}_n^{\uparrow}(\lambda) \le \alpha\bigr\}.
\]
This approach avoids per-sample worst-case inflation, but provides only
asymptotic guarantees \citep[Theorem~A.1]{angelopoulos2024conformal}.

\paragraph{Relationship between the two methods.}
For any $\lambda$,
\[
  \hat{R}_n^{\uparrow}(\lambda)
  = \sup_{t\ge\lambda}\hat{R}_n(t)
  \le
  \frac{1}{n}\sum_{i=1}^n \sup_{s\ge\lambda} L_i(s)
  =
  \tilde{R}_n(\lambda).
\]
Thus risk monotonization produces a pointwise smaller risk curve.
However, the two methods apply different thresholding rules:
loss monotonization uses a finite-sample corrected criterion,
whereas risk monotonization thresholds directly at level $\alpha$.
This leads to a trade-off between conservativeness and theoretical
guarantees.

\subsubsection*{Illustrative example}

To illustrate these differences, consider the synthetic loss
\[
  L_i(\lambda)
  =
  s_i \cdot 0.50 \, e^{-8\lambda}
  +
  h_i
  \exp\!\left(
  -\frac{(\lambda - c_i)^2}{2\sigma_i^2}
  \right)
  +
  \varepsilon_i,
\]
where
\begin{gather*}
s_i \sim \mathrm{Unif}(0.80,1.20),\quad
h_i \sim \mathrm{Unif}(0.06,0.20),\quad
c_i \sim \mathcal{N}(0.42,0.06^2),\\
\sigma_i \sim \mathrm{Unif}(0.04,0.09),\quad
\varepsilon_i \sim \mathcal{N}(0,0.01^2).
\end{gather*}
The exponential component produces a decreasing baseline risk, while
the Gaussian component introduces a localized bump, resulting in a
non-monotonic risk curve.

We consider $n=10{,}000$ calibration samples, $m=100$ candidate
thresholds, and target level $\alpha=0.10$.
Our method selects at level $\alpha'=\alpha-D(m,n)\approx 0.085$.

\begin{figure}[t]
  \centering
  \includegraphics[width=0.6\textwidth]{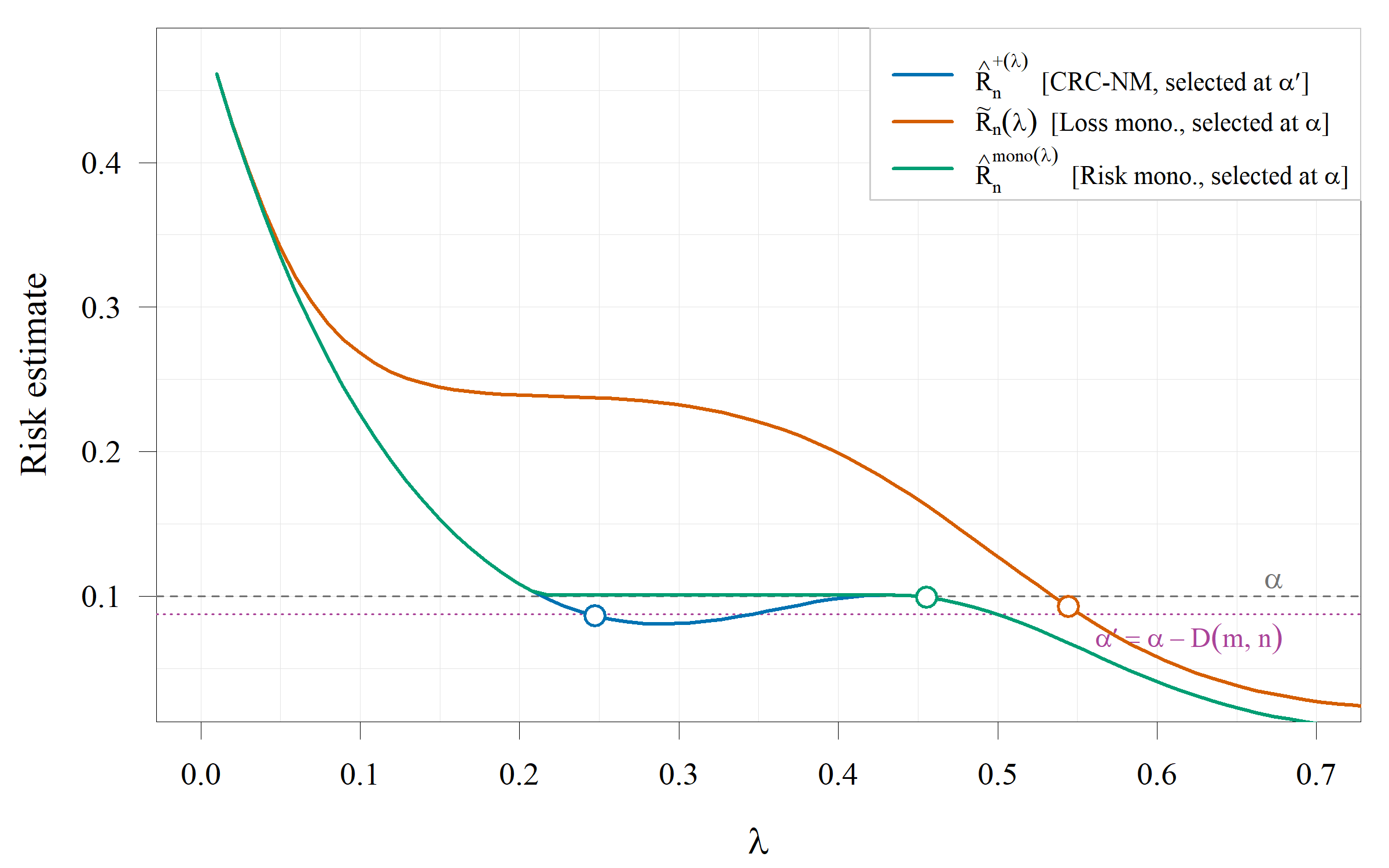}
  \caption{Empirical risk curves and selected thresholds under a
  non-monotonic loss.
  The empirical risk exhibits a localized bump.
  Monotonization methods enforce monotonicity via right-envelope
  corrections, while CRC-NM operates on the original curve and selects
  from feasible regions prior to the bump.}
  \label{fig:supp-monotonization}
\end{figure}

In this example, monotonization propagates the non-monotonic bump
across smaller values of $\lambda$, leading to more conservative
thresholds.
In contrast, CRC-NM incurs a statistical penalty $D(m,n)$ that is
typically smaller than the structural cost introduced by
monotonization.

\subsection{Stability-Based Approaches}
\label{sec:supp-stability}

Stability-based methods \citep{angelopoulos2026conformal} control risk
by comparing a data-driven algorithm $\mathcal{A}$ with an oracle
$\mathcal{A}^*$.
Let $D_{1:n+1}$ be i.i.d.\ data and $D_{-i}$ denote the dataset with
the $i$th observation removed.
The algorithm is $\beta$-stable if
\[
  \mathbb{E}\!\left[
  \frac{1}{n+1}\sum_{i=1}^{n+1}
  \ell(X_i,Y_i;\mathcal{A}(D_{-i}))
  \right]
  \le
  \mathbb{E}\!\left[
  \frac{1}{n+1}\sum_{i=1}^{n+1}
  \ell(X_i,Y_i;\mathcal{A}^*(D_{1:n+1}))
  \right]
  + \beta.
\]
If the oracle satisfies
\[
  \mathbb{E}[\ell(X_{n+1},Y_{n+1};\mathcal{A}^*(D_{1:n+1}))]
  \le \alpha - \beta,
\]
then $\mathcal{A}$ achieves risk control at level $\alpha$.

In practice, $\beta$ is estimated via bootstrap, leading to an adjusted
level $\alpha'=\alpha-\hat{\beta}$ (CRC-C).
For smooth losses and large samples, $\hat{\beta}$ can be small,
yielding near-nominal performance.
However, the estimation of $\beta$ may be sensitive to the bootstrap
procedure and the choice of summary statistic.

\begin{figure}[t]
  \centering
  \includegraphics[width=0.6\textwidth]{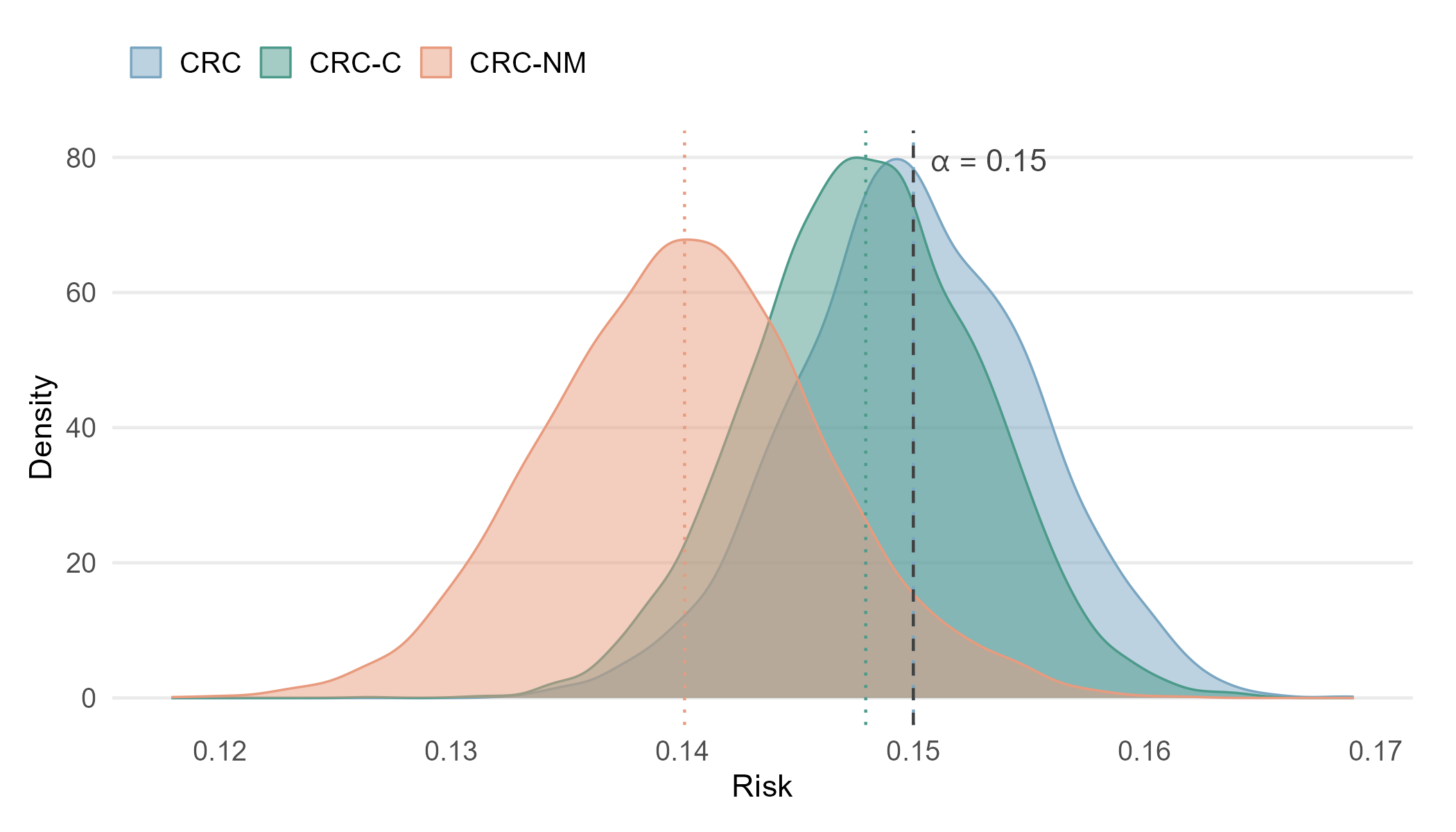}
  \caption{Risk distributions in an illustrative non-monotonic but empirically stable setting.}
  \label{fig:supp-stability}
\end{figure}

Figure~\ref{fig:supp-stability} illustrates a setting where the loss is
non-monotonic but empirically stable.
In this regime, CRC-C provides a small data-adaptive correction,
while CRC-NM applies the explicit finite-sample correction $D(m,n)$.
The latter is typically larger but guarantees uniform control over the
entire parameter grid.

\medskip
\noindent
Overall, these approaches provide complementary strategies for handling
non-monotonicity:
monotonization ensures validity via structural modification,
stability-based methods adapt to data-dependent behavior,
and CRC-NM provides explicit finite-sample guarantees without
additional assumptions.

\section{Additional Experimental Details for Section~\ref{sec:example}}
\label{sec:supp-examples}

This section provides full experimental details, including data
generation, model specifications, and additional explanations for the
results reported in Section~\ref{sec:example}.

\subsection{Synthetic multilabel experiment}
\label{sec:supp-synthetic}

We provide the complete specification of the synthetic
multilabel experiment described in
Section~\ref{sec:synthetic}.

Covariates are generated from
$X\sim\mathcal{N}(0,I_d)$ with $d=15$.
Conditional on $X$, the $K=10$ binary labels
$Y_k$ are generated independently according to
\[
Y_k\mid X
\sim
\operatorname{Bernoulli}
\left(\sigma(X^\top W_k+b_k)\right),
\]
where $\sigma(t)=(1+e^{-t})^{-1}$ is the
logistic function. The entries of $W_k$ and
the intercepts $b_k$ are generated from
$\mathcal{N}(0,0.8^2)$ and
$\mathcal{N}(0,0.2^2)$, respectively.

To construct noisy probability estimates,
independent Gaussian noise with standard
deviation $0.3$ is added to the logits
before applying the logistic function.
The prediction set is defined by
\[
C(X;\lambda)
=
\{k:\hat p_k(X)\ge\lambda\}.
\]
In particular, the terminal threshold
$\lambda_m=1$ produces the empty set.

We consider the non-monotone loss
\[
L(X,Y;\lambda)
=
c
+
w\left(
1-\frac{|Y\cap C(X;\lambda)|}{|C(X;\lambda)|}
\right)
+
a\,s_\lambda
\frac{1+\sin(4\pi s_\lambda)}{2},
\qquad
s_\lambda=\frac{|C(X;\lambda)|}{K},
\]
where precision is defined to be one when
$C(X;\lambda)=\varnothing$.
We set $c=0.10$, $w=0.25$, and $a=0.50$.

Since both the precision error and the
oscillatory penalty lie in $[0,1]$,
the loss satisfies
\[
0.10\le L(X,Y;\lambda)\le0.85.
\]
Thus, the boundedness condition holds
with $B=1$. The oscillatory penalty
introduces non-monotone dependence on
the prediction-set size.

At the terminal threshold $\lambda_m=1$,
the prediction set is empty and
\[
L(X,Y;\lambda_m)=c=0.10
\]
for every observation. Consequently,
\[
R(\lambda_m)=0.10\le\alpha=0.15,
\]
so the population-risk endpoint condition
of Theorem~\ref{thm:nonmonotonic} is satisfied.
Moreover, the stronger almost-sure
endpoint condition also holds.

We use a calibration sample of size
$n_{\mathrm{cal}}=10{,}000$, a test sample
of size $n_{\mathrm{test}}=500$, and a
selection grid containing $m=100$ points.
The target risk level is $\alpha=0.15$,
and results are summarized over $400$
independent repetitions.

For CRC-C, the stability correction
$\hat\beta$ is estimated as the $90$th
percentile of the absolute bootstrap
deviations based on $200$ resamples.
CRC-NM uses the explicit correction
$D(m,n)$ from
Theorem~\ref{thm:nonmonotonic}.
For the chosen sample size and grid,
$R(\lambda_m)\le\alpha-D(m,n)$,
so the sufficient conditions for the
exact expected-risk guarantee of the
adjusted CRC-NM procedure are satisfied.

\begin{table}[h]
\centering
\caption{Synthetic multilabel experiment:
mean empirical test risk, empirical test-risk
exceedance frequency, and mean prediction-set
size ($\alpha=0.15$, $400$ repetitions).}
\begin{tabular}{lccc}
\hline
Method & Mean risk & Exceedance (\%) & Mean size \\
\hline
CRC        & 0.147 & 10.5 & 1.26 \\
CRC-C      & 0.144 & 3.2  & 1.14 \\
Loss Mono  & 0.142 & 0.0  & 1.07 \\
Risk Mono  & 0.147 & 10.5 & 1.26 \\
CRC-NM     & 0.128 & 0.0  & 0.64 \\
\hline
\end{tabular}
\label{tab:supp-synthetic}
\end{table}

Table~\ref{tab:supp-synthetic} reports the
detailed results summarized in the main text.
Standard CRC and risk monotonization exceed
the target risk level in approximately $10.5\%$
of repetitions, while CRC-C reduces this
frequency to approximately $3.2\%$.
No empirical test-risk exceedances are
observed for loss monotonization or CRC-NM
across the $400$ repetitions.

CRC-NM achieves the lowest mean empirical
test risk, $0.128$, while producing the
smallest mean prediction-set size, $0.64$.
These findings illustrate that, under
non-monotone losses, improved empirical
risk control need not come at the expense
of larger prediction sets.

\subsection{COCO Object Detection}
\label{sec:supp-coco}

\paragraph{Dataset and model.}
We evaluate the proposed methods on the COCO 2017
validation set using a pretrained Faster R-CNN detector
with a ResNet-50-FPN backbone. The detector produces
candidate bounding boxes and associated confidence
scores for each image.

\paragraph{Prediction sets.}
For each image $X$, prediction sets are constructed
by thresholding detection confidence scores.
For a tuning parameter $\lambda\in[0,1]$, define
\[
C(X;\lambda)
=
\{\text{detections with confidence score}
\ge\lambda\}.
\]
Thus, larger values of $\lambda$ correspond to
smaller prediction sets.

\paragraph{Matching procedure.}
Selected detections are matched to ground-truth
objects using greedy one-to-one matching with
an intersection-over-union (IoU) threshold of $0.3$.
Each detection and each ground-truth object can
participate in at most one match.

Let $n_{\mathrm{matched}}$ denote the number of
matched ground-truth objects,
$n_{\mathrm{gt}}$ the total number of ground-truth
objects, and $|C(X;\lambda)|$ the number of
selected detections.

\paragraph{Loss function.}
We consider the loss
\[
L(X,Y;\lambda)
=
(1-\gamma)
\left(
1-\frac{n_{\mathrm{matched}}}{n_{\mathrm{gt}}}
\right)
+
\gamma\,\phi(|C(X;\lambda)|),
\]
where
\[
\phi(s)
=
\min\left\{
\frac{(s-K_0)_+}{\tau},\,1
\right\}.
\]
We set $\gamma=0.35$, $K_0=3$, and $\tau=5$.

The first component measures a recall-based error,
while the second penalizes large prediction sets.
The set-size penalty is zero when the number of
detections does not exceed $K_0$, increases
linearly over the next $\tau$ detections, and
is capped at one. The competing effects of
these two components can produce a non-monotone
risk function as the detection threshold varies.
The loss is bounded in $[0,1]$.

\paragraph{Experimental setup.}
We use a uniform grid of $m=200$ candidate
thresholds in $[0.02,0.75]$.
Each repetition uses
$n_{\mathrm{cal}}=2{,}500$ calibration images
and $n_{\mathrm{test}}=800$ test images,
with a target risk level of $\alpha=0.33$.
Results are summarized over $3{,}000$
random splits.

For CRC-C, the stability correction
$\hat\beta$ is estimated as the $90$th
percentile of absolute bootstrap deviations
based on $200$ resamples.
CRC-NM uses the explicit correction
$D(m,n)$ from
Theorem~\ref{thm:nonmonotonic}.

\paragraph{Terminal-risk condition.}
The empirical mean loss at the terminal
grid point $\lambda_m$ is approximately
$0.32$, below the target level $\alpha=0.33$.
This is consistent with the population-risk
condition
\[
R(\lambda_m)\le\alpha
\]
in Theorem~\ref{thm:nonmonotonic}, although
the empirical estimate alone does not
establish this condition.

Unlike the synthetic experiment, the
almost-sure endpoint condition
$L_i(\lambda_m)\le\alpha$ need not hold
for every observation.

For the adjusted CRC-NM procedure, the
exact expected-risk guarantee requires
the stronger condition
\[
R(\lambda_m)\le\alpha-D(m,n).
\]
This condition is not established in
the present experiment. Therefore,
the observed performance of the adjusted
procedure should be interpreted as
empirical evidence rather than a direct
verification of its exact theoretical
guarantee.

\paragraph{Discussion.}
The loss exhibits non-monotone behavior
due to the trade-off between detection
recall and prediction-set size.
Standard CRC and risk monotonization
exceed the target empirical test-risk
level in approximately $48\%$ of random
splits, compared with approximately
$45\%$ for CRC-C.

Loss monotonization reduces the
exceedance frequency to approximately
$2\%$, but produces substantially
larger prediction sets, with a mean
size of $12.6$ compared with $5.1$
for standard CRC.

CRC-NM, using the adjusted level
$\alpha'=\alpha-D(m,n)$, reduces the
empirical test-risk exceedance frequency
to approximately $1\%$, while maintaining
a mean prediction-set size of $8.7$.
Its mean empirical test risk is
approximately $0.31$, below the
target level $\alpha=0.33$.

These results illustrate that CRC-NM
can substantially improve empirical
risk control under non-monotone losses
while producing smaller prediction sets
than loss monotonization, without
modifying the original loss function.

\section{Tighter Uniform Concentration Bounds via Variance-Sensitive Inequalities}
\label{appendix:tighter_bounds}
 
The uniform concentration bound in Lemma~\ref{lemma:uniform} controls $\mathbb{E}[\sup_{\lambda\in\Lambda}|\hat R_n(\lambda)-R(\lambda)|]$
using Hoeffding's inequality, yielding a leading term of order
$B\sqrt{\log(2m)/(2n)}$.
This bound depends only on the range~$B$ and ignores the
variance of the losses entirely.
When $\operatorname{Var}(L_i(\lambda))\ll B^2$ for most or all
$\lambda\in\Lambda$, as is common in calibration and risk-control
settings, variance-sensitive concentration inequalities can deliver
substantially tighter bounds.
 
We present two refinements:
a \emph{Bernstein bound} that replaces the range~$B$ in the leading term
with the worst-case standard deviation~$\sigma_{\max}$
(Appendix~\ref{appendix:bernstein_bound}),
and an \emph{empirical Bernstein bound} that replaces $\sigma_{\max}$
with a data-driven estimate, requiring no knowledge of population quantities
(Appendix~\ref{appendix:emp_bernstein_bound}).
We then propagate the improved bound through Theorem~\ref{thm:nonmonotonic}
(Appendix~\ref{appendix:propagation})
and provide a numerical comparison across variance regimes
(Appendix~\ref{appendix:numerical}).
 
\subsection{Bernstein-Type Uniform Concentration}
\label{appendix:bernstein_bound}

\begin{lemma}[Bernstein uniform concentration bound]\label{lemma:bernstein}
Let \(\Lambda=\{\lambda_1,\dots,\lambda_m\}\) be a finite parameter set with
\(|\Lambda|=m<\infty\). For each \(\lambda\in\Lambda\), let
\(L_1(\lambda),\dots,L_n(\lambda)\) be i.i.d.\ random variables satisfying
\[
0\le L_i(\lambda)\le B \qquad \text{a.s.}
\]
Define
\[
R(\lambda)=\mathbb{E}[L_1(\lambda)], \qquad
\hat R_n(\lambda)=\frac{1}{n}\sum_{i=1}^n L_i(\lambda), \qquad
\sigma^2(\lambda)=\operatorname{Var}(L_1(\lambda)),
\]
and set
\[
\sigma_{\max}^2=\max_{\lambda\in\Lambda}\sigma^2(\lambda).
\]
Then
\[
\mathbb{E}\!\left[
\sup_{\lambda\in\Lambda}
\bigl|\hat R_n(\lambda)-R(\lambda)\bigr|
\right]
\;\le\;
\sigma_{\max}\sqrt{\frac{2\log(2m)}{n}}
\;+\;
\frac{B\log(2m)}{3n}.
\]
\end{lemma}

\begin{proof}
For each \(\lambda\in\Lambda\), define
\[
Y_\lambda := \hat R_n(\lambda)-R(\lambda)
= \frac{1}{n}\sum_{i=1}^n \bigl(L_i(\lambda)-\mathbb{E}[L_i(\lambda)]\bigr).
\]
Then \(\mathbb{E}[Y_\lambda]=0\), and
\[
Z:=\sup_{\lambda\in\Lambda}|Y_\lambda|
=\max\{Y_\lambda,\,-Y_\lambda:\lambda\in\Lambda\}.
\]

Let \(W_i(\lambda):=L_i(\lambda)-\mathbb{E}[L_i(\lambda)]\). Since
\(0\le L_i(\lambda)\le B\) a.s., we have $|W_i(\lambda)|\le B$ a.s. and
\(\operatorname{Var}(W_i(\lambda))=\sigma^2(\lambda)\).
By Bernstein's mgf bound for bounded centered variables
(see, e.g., \cite[Theorem 2.10]{boucheron2013concentration}),
for all \(\eta\in(0,3n/B)\),
\[
\log \mathbb{E}\!\left[e^{\eta Y_\lambda}\right]
\le
\frac{\eta^2 \sigma^2(\lambda)/n}{2\bigl(1-\eta B/(3n)\bigr)}.
\]
The same bound holds for \(-Y_\lambda\). Hence every element in
\[
\mathcal{Y}:=\{Y_\lambda,-Y_\lambda:\lambda\in\Lambda\}
\]
is sub-gamma with variance factor \(v=\sigma_{\max}^2/n\) and scale
parameter \(c=B/(3n)\).

Applying a standard maximal inequality for sub-gamma variables (see,
e.g., \cite[Theorem 2.5]{boucheron2013concentration}) to the
collection \(\mathcal{Y}\) of size \(2m\), we obtain
\[
\mathbb{E}[Z]
\le
\sqrt{2v\log(2m)} + c\log(2m).
\]
Substituting \(v=\sigma_{\max}^2/n\) and \(c=B/(3n)\) gives
\[
\mathbb{E}\!\left[
\sup_{\lambda\in\Lambda}
\bigl|\hat R_n(\lambda)-R(\lambda)\bigr|
\right]
\le
\sigma_{\max}\sqrt{\frac{2\log(2m)}{n}}
+
\frac{B\log(2m)}{3n}.
\]
\end{proof}
 
\begin{remark}[Variance ratio gain]\label{remark:variance_ratio}
The ratio of the Bernstein leading term to the Hoeffding leading term is
\[
\frac{\sigma_{\max}\sqrt{2\log(2m)/n}}{B\sqrt{\log(2m)/(2n)}}
\;=\;
\frac{2\sigma_{\max}}{B}.
\]
When $\sigma_{\max}\ll B$ (for instance, if losses are concentrated
near~$0$ or well below~$B$), this ratio can be much smaller than~$1$.
\end{remark}
 
\subsection{Empirical Bernstein Uniform Concentration}
\label{appendix:emp_bernstein_bound}

The Bernstein bound of Lemma~\ref{lemma:bernstein} requires knowledge
of $\sigma_{\max}^2$. Following \cite{maurer2009empirical}, one may
instead obtain a high-probability uniform bound using the empirical variance
\[
\hat\sigma_n^2(\lambda)
=\frac{1}{n}\sum_{i=1}^n
\bigl(L_i(\lambda)-\hat R_n(\lambda)\bigr)^2.
\]

\begin{lemma}[Empirical Bernstein uniform concentration]
\label{lemma:emp_bernstein}
Assume the same setup as in Lemma~\ref{lemma:bernstein}, with $n\ge2$,
and define
\[
\hat\sigma_{\max}=\max_{\lambda\in\Lambda}\hat\sigma_n(\lambda).
\]
Then for every $\delta\in(0,1)$, with probability at least $1-\delta$,
\[
\sup_{\lambda\in\Lambda}
\bigl|\hat R_n(\lambda)-R(\lambda)\bigr|
\le
\hat\sigma_{\max}
\sqrt{\frac{2\log(2m/\delta)}{n-1}}
+\frac{7B\log(2m/\delta)}{3(n-1)}.
\]
\end{lemma}

\begin{proof}
For each $\lambda\in\Lambda$, the empirical Bernstein inequality
of \cite{maurer2009empirical}, applied to $L_i(\lambda)/B$,
gives a two-sided deviation bound with confidence parameter $\delta/m$.
The variance appearing in that inequality is the usual sample variance,
\[
S_n^2(\lambda)=\frac{1}{n-1}\sum_{i=1}^n
\bigl(L_i(\lambda)-\hat R_n(\lambda)\bigr)^2
=\frac{n}{n-1}\hat\sigma_n^2(\lambda).
\]
Consequently, for each fixed $\lambda$, except on an event of
probability at most $\delta/m$,
\[
\bigl|\hat R_n(\lambda)-R(\lambda)\bigr|
\le
\hat\sigma_n(\lambda)
\sqrt{\frac{2\log(2m/\delta)}{n-1}}
+\frac{7B\log(2m/\delta)}{3(n-1)}.
\]
A union bound over the $m$ grid points and the definition of
$\hat\sigma_{\max}$ establish the result.
\end{proof}

Unlike Lemma~\ref{lemma:bernstein}, this is a high-probability
inequality with a random right-hand side, not a deterministic
expectation bound. It provides a computable uniform-deviation
certificate, but substituting that certificate directly into the
adjusted CRC rule does not, by itself, prove exact expected-risk
control. Such a guarantee requires an additional treatment of the
probability-$\delta$ failure event.


\subsection{Propagation to the Main Theorem}
\label{appendix:propagation}

The Bernstein uniform-concentration bound in
Lemma~\ref{lemma:bernstein} provides a refinement
of Theorem~\ref{thm:nonmonotonic} by incorporating
the variability of the losses.

\begin{corollary}[Bernstein-type risk bound]
\label{cor:bernstein_theorem}
Suppose that the assumptions of
Theorem~\ref{thm:nonmonotonic} hold.
Then
\[
\mathbb E[L_{n+1}(\hat\lambda)]
\le
\alpha+
\min\left\{
D(m,n),D_{\mathrm{Bern}}(m,n)
\right\},
\]
where
\[
D_{\mathrm{Bern}}(m,n)
=
\sigma_{\max}
\sqrt{\frac{2\log(2m)}{n}}
+
\frac{B\log(2m)}{3n},
\]
and
\[
\sigma_{\max}^2
=
\max_{\lambda\in\Lambda}
\operatorname{Var}(L_i(\lambda)).
\]

Furthermore, let
\[
\alpha'_{\mathrm{Bern}}
=
\alpha-
\min\left\{
D(m,n),D_{\mathrm{Bern}}(m,n)
\right\}>0.
\]
If
\[
R(\lambda_m)\le\alpha'_{\mathrm{Bern}},
\]
then the corresponding adjusted selector
$\hat\lambda^{\mathrm{adj}}_{\mathrm{Bern}}$,
defined by replacing $\alpha$ with
$\alpha'_{\mathrm{Bern}}$ in the original
selection rule, satisfies
\[
\mathbb E\left[
L_{n+1}(\hat\lambda^{\mathrm{adj}}_{\mathrm{Bern}})
\right]
\le\alpha.
\]
\end{corollary}

\begin{proof}
Let
\[
\Delta_n
=
\max_{\lambda\in\Lambda}
|\hat R_n(\lambda)-R(\lambda)|.
\]
As established in the proof of
Theorem~\ref{thm:nonmonotonic},
the population-risk endpoint condition
$R(\lambda_m)\le\alpha$ implies
\[
R(\hat\lambda)\le\alpha+\Delta_n
\quad\text{almost surely},
\]
including when the feasible set is empty.

Taking expectations and applying
Lemmas~\ref{lemma:uniform} and
\ref{lemma:bernstein} yields
\[
\begin{aligned}
\mathbb E[L_{n+1}(\hat\lambda)]
&=\mathbb E[R(\hat\lambda)]\\
&\le\alpha+\mathbb E[\Delta_n]\\
&\le\alpha+
\min\left\{
D(m,n),D_{\mathrm{Bern}}(m,n)
\right\}.
\end{aligned}
\]

The adjusted guarantee follows by applying
the same argument at level
$\alpha'_{\mathrm{Bern}}$, using the
additional condition
$R(\lambda_m)\le\alpha'_{\mathrm{Bern}}$.
\end{proof}

The Bernstein-type correction can be smaller
than the Hoeffding-based correction when
the losses have sufficiently low variance.
Consequently, combining the two bounds
provides a potentially sharper finite-sample
risk guarantee, while retaining the weaker
population-risk endpoint condition of
Theorem~\ref{thm:nonmonotonic}.

\subsection{Numerical Comparison}
\label{appendix:numerical}

We compare the two expected uniform-deviation bounds and one
high-probability uniform-deviation bound in a regime representative of
typical calibration settings: sample sizes
$n\in\{1000,\,2000,\,5000,\,10{,}000,\,20{,}000\}$,
$|\Lambda|=m=200$ candidate hyperparameter values,
and $B=1$.
Three variance regimes are considered:
$\sigma_{\max}/B\in\{0.1,\;0.3,\;0.5\}$.
For the empirical Bernstein bound, we set $\delta=0.05$ and
substitute $\hat\sigma_{\max}=\sigma_{\max}$ for illustration.
The resulting empirical Bernstein values are not realized
data-driven certificates; the substitution isolates the formulas'
variance dependence. In particular, their confidence interpretation
differs from that of the two expectation bounds.
We also examine the sensitivity to $m$ at fixed $n=5000$.

\begin{table}[ht]
\centering
\caption{Numerical values of uniform-deviation bounds for $B=1$, $m=200$.
H\,=\,Hoeffding (Lemma~\ref{lemma:uniform}),
B\,=\,Bernstein (Lemma~\ref{lemma:bernstein}),
EB\,=\,Empirical Bernstein (Lemma~\ref{lemma:emp_bernstein},
$\delta=0.05$, oracle variance substitution).
H and B bound expected deviations; EB bounds deviations with probability $0.95$.}
\label{tab:comparison}
\smallskip
\begin{tabular}{@{}r ccc ccc ccc@{}}
\hline
& \multicolumn{3}{c}{$\sigma_{\max}/B=0.1$}
& \multicolumn{3}{c}{$\sigma_{\max}/B=0.3$}
& \multicolumn{3}{c}{$\sigma_{\max}/B=0.5$} \\
\hline
$n$ & H & B & EB & H & B & EB & H & B & EB \\
\hline
  1{,}000 & .059 & .013 & .034 & .059 & .035 & .061 & .059 & .057 & .088 \\[2pt]
  2{,}000 & .042 & .009 & .020 & .042 & .024 & .039 & .042 & .040 & .058 \\[2pt]
  5{,}000 & .027 & .005 & .010 & .027 & .015 & .022 & .027 & .025 & .034 \\[2pt]
 10{,}000 & .019 & .004 & .006 & .019 & .011 & .015 & .019 & .018 & .023 \\[2pt]
 20{,}000 & .013 & .003 & .004 & .013 & .007 & .010 & .013 & .012 & .016 \\
\hline
\end{tabular}
\end{table}

\begin{table}[ht]
\centering
\caption{Sensitivity to $|\Lambda|=m$ at fixed $n=5{,}000$,
$B=1$, $\sigma_{\max}/B=0.3$, $\delta=0.05$;
EB uses the oracle variance substitution.}
\label{tab:sensitivity_m}
\smallskip
\begin{tabular}{@{}r ccc@{}}
\hline
$m$ & Hoeffding & Bernstein & Emp.\ Bernstein \\
\hline
  50 & .024 & .013 & .020 \\[2pt]
 100 & .025 & .014 & .021 \\[2pt]
 200 & .027 & .015 & .022 \\[2pt]
 500 & .028 & .016 & .024 \\
\hline
\end{tabular}
\end{table}

\begin{figure}[ht]
\centering
\begin{tikzpicture}
\begin{axis}[
  width=0.65\textwidth,
  height=7.5cm,
  xlabel={Sample size $n$},
  ylabel={Uniform-deviation bound value},
  title={$B=1$, \ $m=200$, \ $\sigma_{\max}/B=0.3$},
  xmin=500, xmax=20000,
  ymin=0, ymax=0.07,
  legend style={at={(0.97,0.97)}, anchor=north east, font=\small},
  grid=major,
  grid style={gray!30},
  every axis plot/.append style={thick},
  tick label style={font=\small},
  label style={font=\small},
  title style={font=\normalsize},
]

\addplot[color=hoeffding_col, domain=500:20000, samples=100]
  {sqrt(5.9915/(2*x)) + 1/(2*sqrt(2*x*5.9915))};
\addlegendentry{Hoeffding (Lemma~\ref{lemma:uniform})}

\addplot[color=bernstein_col, dashed, domain=500:20000, samples=100]
  {0.3*sqrt(2*5.9915/x) + 5.9915/(3*x)};
\addlegendentry{Bernstein (Lemma~\ref{lemma:bernstein})}

\addplot[color=empbern_col, dashdotted, domain=500:20000, samples=100]
  {0.3*sqrt(2*8.9872/(x-1)) + 7*8.9872/(3*(x-1))};
\addlegendentry{Emp.\ Bernstein (Lemma~\ref{lemma:emp_bernstein})}

\end{axis}
\end{tikzpicture}
\caption{Uniform-deviation bound values as a function of sample size for
$\sigma_{\max}/B=0.3$ and $m=200$.
The Bernstein bound (dashed) improves substantially over
Hoeffding (solid) throughout the range $n\in[500,\,20{,}000]$.
The empirical Bernstein values (dash-dotted, $\delta=0.05$,
oracle variance substitution) track the oracle Bernstein values at large $n$, with a larger
gap at small $n$ due to the variance estimation penalty and the
additional $\log(1/\delta)$ factor.}
\label{fig:bounds_vs_n}
\end{figure}

\begin{figure}[ht]
\centering
\begin{tikzpicture}
\begin{axis}[
  width=0.65\textwidth,
  height=7.5cm,
  xlabel={Variance ratio $\sigma_{\max}/B$},
  ylabel={Uniform-deviation bound value},
  title={$B=1$, \ $m=200$, \ $n=5{,}000$},
  xmin=0.01, xmax=0.5,
  ymin=0, ymax=0.04,
  legend style={at={(0.97,0.97)}, anchor=north east, font=\small},
  grid=major,
  grid style={gray!30},
  every axis plot/.append style={thick},
  tick label style={font=\small},
  label style={font=\small},
  title style={font=\normalsize},
]

\addplot[color=hoeffding_col, domain=0.01:0.5, samples=80]
  {0.02652};
\addlegendentry{Hoeffding (Lemma~\ref{lemma:uniform})}

\addplot[color=bernstein_col, dashed, domain=0.01:0.5, samples=80]
  {x*0.04895 + 0.000399};
\addlegendentry{Bernstein (Lemma~\ref{lemma:bernstein})}

\addplot[color=empbern_col, dashdotted, domain=0.01:0.5, samples=80]
  {x*sqrt(2*8.9872/4999) + 0.004195};
\addlegendentry{Emp.\ Bernstein (Lemma~\ref{lemma:emp_bernstein})}

\end{axis}
\end{tikzpicture}
\caption{Uniform-deviation bound values as a function of the variance ratio
$\sigma_{\max}/B$ at fixed $n=5{,}000$ and $m=200$.
The Hoeffding bound (solid horizontal line) is variance-blind.
Both Bernstein bounds decrease linearly in $\sigma_{\max}/B$,
offering the largest gains in low-variance regimes.
The empirical Bernstein formula ($\delta=0.05$, oracle variance
substitution) has a steeper slope
and larger intercept than the oracle Bernstein bound, reflecting
the additional $\log(1/\delta)$ factor.}
\label{fig:bounds_vs_sigma}
\end{figure}

\subsubsection*{Discussion}

Several observations emerge from the numerical comparison
(Tables~\ref{tab:comparison}--\ref{tab:sensitivity_m},
Figures~\ref{fig:bounds_vs_n} and~\ref{fig:bounds_vs_sigma}).

First, in this large-$n$, moderate-$m$ regime, the Bernstein bound
dominates the Hoeffding bound across all configurations.
At $\sigma_{\max}/B=0.1$ and $n=5{,}000$, the Bernstein bound is
roughly $5$ times smaller ($0.005$ vs.\ $0.027$),
permitting a smaller deterministic correction in the adjusted
selection level and a correspondingly less restrictive
selection rule.
Even at $\sigma_{\max}/B=0.3$, the Bernstein bound is approximately
$44\%$ tighter than Hoeffding at $n=5{,}000$.

Second, the empirical Bernstein numerical values are larger than
the oracle Bernstein values in these configurations, owing to two factors:
the additional $\log(1/\delta)$ term (since EB is a high-probability
bound evaluated at $\delta=0.05$) and the larger constant $7/3$
in the second-order term.
At small sample sizes, this gap can be substantial: at $n=1{,}000$
with $\sigma_{\max}/B=0.5$, the empirical Bernstein value ($0.088$)
exceeds both the Hoeffding ($0.059$) and oracle Bernstein ($0.057$)
bounds.
However, the gap narrows rapidly with $n$, and by $n=10{,}000$
the empirical Bernstein bound is competitive across all variance
regimes.
The numerical minimum of an expectation bound and a random
high-probability bound cannot be treated as an exact
expected-risk correction without a separate argument.
For an adaptive procedure, the failure event of the empirical
Bernstein certificate must be accounted for explicitly.

Third, Table~\ref{tab:sensitivity_m} shows that the bounds are
relatively insensitive to $m$ over the range $m\in\{50,\ldots,500\}$
at $n=5{,}000$, owing to the $\sqrt{\log m}$ dependence.
Doubling $m$ from $100$ to $200$ increases the Hoeffding bound by
only $8\%$ ($0.025\to 0.027$) and the Bernstein bound by a comparable
amount ($0.014\to 0.015$).
This is reassuring for applications where the hyperparameter grid
is moderately large.

Fourth, for losses in $[0,B]$, Popoviciu's inequality gives
$\sigma_{\max}/B\le1/2$. Thus the displayed range up to $1/2$
includes the largest possible standard-deviation ratio; the regime
$\sigma_{\max}/B\to1$ is not admissible. In the displayed
configurations, the variance-sensitive expectation bound is most
valuable when $\sigma_{\max}/B$ is small. The empirical Bernstein
certificate remains a high-probability statement and should not be
interpreted as an exact expected-risk correction.
 
\end{appendix}

\end{document}